\documentclass{article} 
\usepackage{iclr2026_conference,times}

\usepackage{amsmath,amsfonts,bm}

\def\eqref#1{equation~\ref{#1}}

\def\plaineqref#1{\ref{#1}}

\def\1{\bm{1}}

\DeclareMathAlphabet{\mathsfit}{\encodingdefault}{\sfdefault}{m}{sl}
\SetMathAlphabet{\mathsfit}{bold}{\encodingdefault}{\sfdefault}{bx}{n}

\usepackage{hyperref}
\usepackage{url}
\usepackage{graphicx}
\usepackage{wrapfig}
\usepackage{subcaption}
\usepackage{amssymb}
\usepackage{amsthm}
\usepackage{booktabs}
\usepackage{makecell}
\usepackage{enumitem}
\usepackage[most]{tcolorbox}
\usepackage{bbding}
\usepackage{colortbl}
\usepackage{algpseudocode}
\makeatletter
\newcommand{\AlgCenter}[1]{\Statex \hspace*{-\ALG@tlm}\parbox{\dimexpr\linewidth+\ALG@tlm\relax}{\centering #1\par}}
\makeatother

\newcommand{\echotitleicon}{%
  \raisebox{-0.4em}{%
    \includegraphics[height=2em,trim=100 58 101 168,clip]{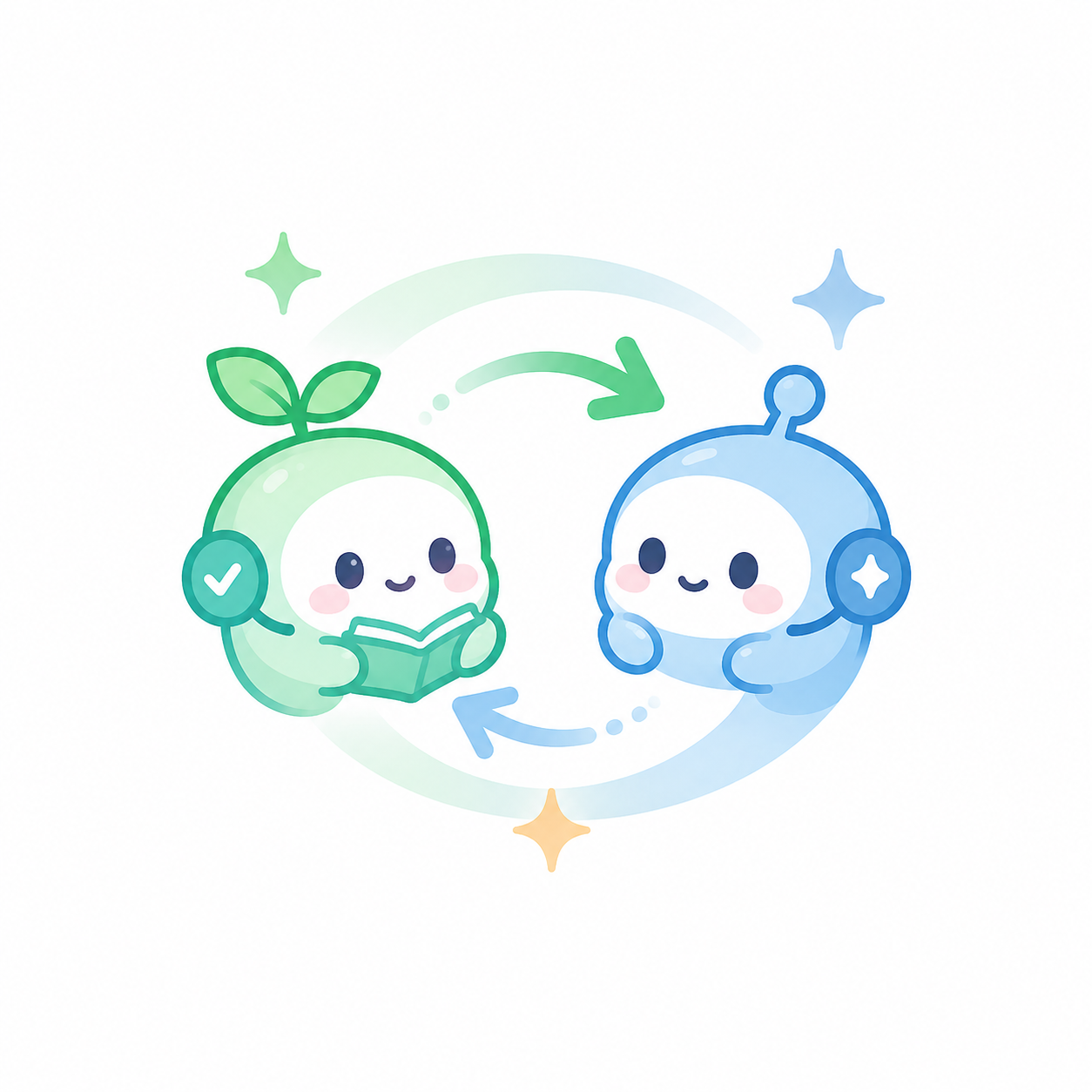}%
  }%
  \ %
}
\newcommand{\ucobletter}[1]{\underline{#1}}
\definecolor{scigreen}{RGB}{0, 128, 96}
\definecolor{scired}{RGB}{190, 80, 70}
\definecolor{sciblue}{RGB}{55, 105, 180}
\definecolor{ucobblue}{RGB}{236, 245, 255}
\definecolor{succyellow}{RGB}{255, 253, 238}
\definecolor{ucobmetric}{RGB}{242, 249, 238}
\definecolor{groupgray}{RGB}{242, 242, 242}
\newcolumntype{Y}{>{\columncolor{succyellow}}c}
\newcommand{\cmark}{\textcolor{scigreen}{\checkmark}}
\newcommand{\xmark}{\textcolor{scired}{\(\times\)}}
\newcommand{\pmark}{\textcolor{sciblue}{\(\triangle\)}}
\newcounter{ucobalgorithm}
\newenvironment{ucobalgorithm}[1]{%
  \refstepcounter{ucobalgorithm}%
  \par\vspace{0.5em}\noindent
  \begin{minipage}{\textwidth}
  \hrule\vspace{0.45em}
  \noindent\textbf{Algorithm~\theucobalgorithm. #1}
  \vspace{0.45em}\hrule\vspace{0.35em}
}{%
  \vspace{0.2em}\hrule
  \end{minipage}
  \par\vspace{0.7em}
}
\newtheorem{proposition}{Proposition}
\newtheorem{corollary}{Corollary}

\title{\texorpdfstring{\hyphenpenalty=10000\exhyphenpenalty=10000\echotitleicon\textbf{UCOB}: Learning to \ucobletter{U}tilize and Evolve Agentic Skills via \ucobletter{C}redit-Aware \ucobletter{O}n-Policy \ucobletter{B}idirectional Self-Distillation}{UCOB: Learning to Utilize and Evolve Agentic Skills via Credit-Aware On-Policy Bidirectional Self-Distillation}}

\author{%
   \textbf{Songjun Tu$^\spadesuit$}, 
   \textbf{Chengdong Xu$^\clubsuit$}, 
   \textbf{Qichao Zhang$^\spadesuit$\Envelope}, 
   \textbf{Yiwen Ma$^\spadesuit$},
   \textbf{Yaocheng Zhang$^\spadesuit$}, \\
   \textbf{Linjing Li$^\spadesuit$}, 
   \textbf{Dong Li$^\diamondsuit$},
   \textbf{Xiangyuan Lan$^\clubsuit$},
   \textbf{Dongbin Zhao$^\spadesuit$}\\
   $\spadesuit$ Institute of Automation, Chinese Academy of Sciences\\
   $\clubsuit$ Pengcheng Laboratory \quad
   $\diamondsuit$ Memorax AI\\
   \texttt{\{tusongjun2023,zhangqichao2014\}@ia.ac.cn}\\
   \underline{\textbf{\textit{June 28, 2026}}}
}

\iclrfinalcopy 
\begin{document}

\maketitle
\vspace{-1em}

\begin{abstract}
Skills can improve agentic reinforcement learning by reusing past experience as textual guidance, but retrieved skills are not always reliable: they may help in one state while misleading in another.
This makes the common privileged-teacher assumption fragile, namely that a skill-conditioned prompt can be treated as a fixed teacher for the no-skill prompt.
We propose \textsc{UCOB}, a framework that first learns when to utilize agentic skills via credit-aware on-policy bidirectional self-distillation, then evolves them through utility-aware skill updates and reflection self-training.
\textsc{UCOB} constructs two on-policy context views from skill-conditioned and no-skill prompts, compares return-to-go among rollouts sharing the same task and anchor state, and uses the higher-return view as the local teacher.
The resulting local credit signal internalizes useful skill-conditioned behavior, corrects misleading skill usage, and further updates the dual-granularity skill memory, informs utility-aware retrieval, and supports reflection self-training.
Experiments on \textsc{ALFWorld}, \textsc{WebShop}, and \textsc{Search-QA} show that \textsc{UCOB} outperforms skill-free RL baselines, skill-augmented methods, and self-distillation baselines across model scales, achieving up to \textbf{23.5} and \textbf{18.0} absolute success-rate gains over SOTA baselines on \textsc{ALFWorld} and \textsc{WebShop}, respectively.
Ablations and analyses further validate its core mechanisms, continual adaptation across environments, and modest training overhead.
\begin{center}
Code is available at \url{https://github.com/TU2021/UCOB}.
\end{center}
\end{abstract}

\section{Introduction}

\begin{figure}[t]
    \centering
    \includegraphics[width=\textwidth]{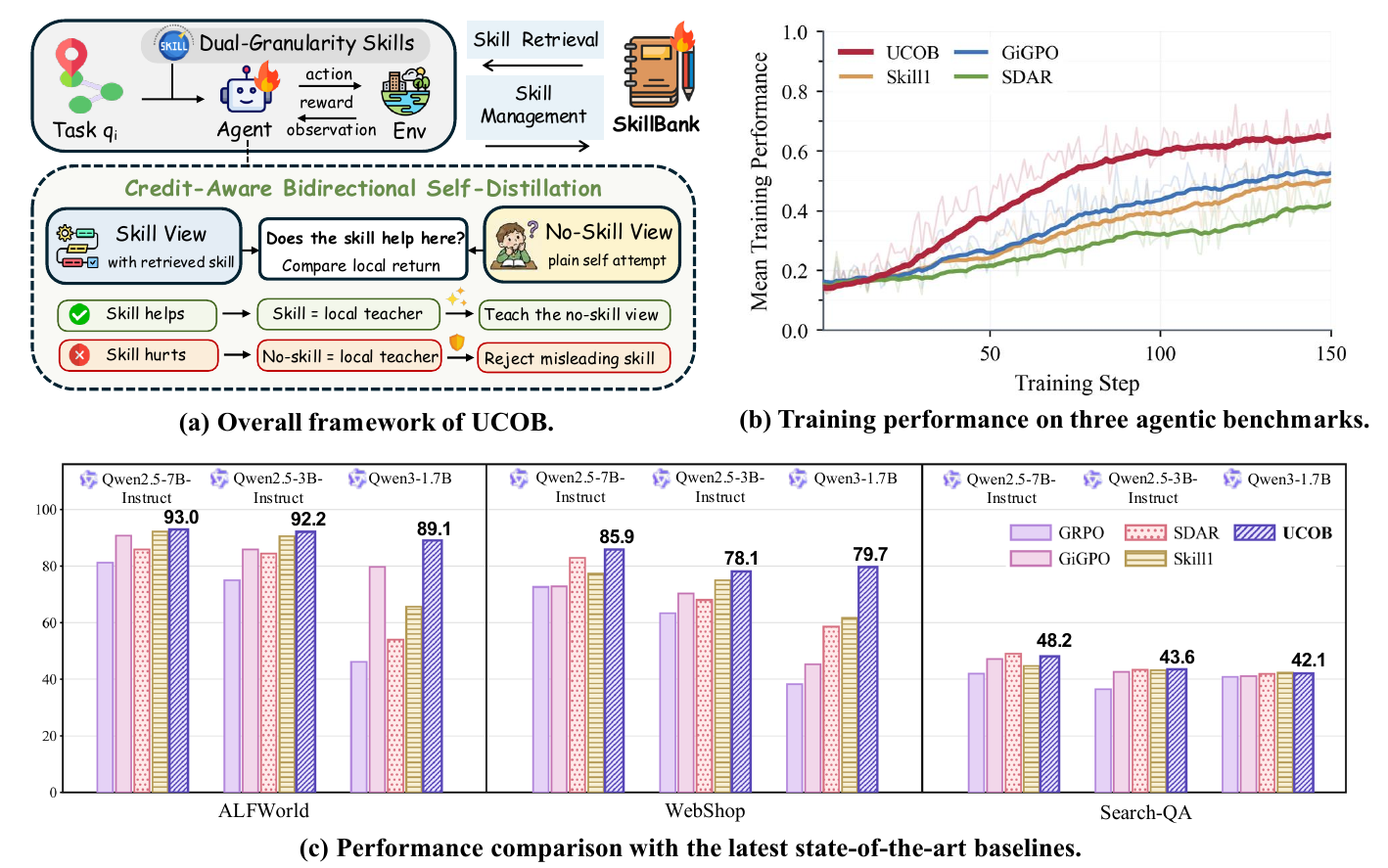}
    \caption{Overview and empirical summary of \textsc{UCOB}.}
    \label{fig:overall}
    \vspace{-1em}
\end{figure}

Agentic reinforcement learning (RL) has become a central paradigm for improving language agents in long-horizon interactive tasks.
Unlike single-turn reasoning, these environments require agents to act under partial observations, recover from local mistakes, and assign sparse outcome rewards to many intermediate decisions.
A natural way to reduce this difficulty is to reuse past experience: agents can reflect on trajectories, store reusable textual skills, retrieve relevant skills in later states, and condition future actions on this external memory~\citep{shinn2023reflexionlanguageagentsverbal,wang2023voyageropenendedembodiedagent}.
Recent skill-augmented agent methods show that such memory can improve exploration and sample efficiency~\citep{tu2026dynamicdualgranularityskillbank,shi2026skill1unifiedevolutionskillaugmented}.
When these retrieved skills are further used for training, however, a stronger assumption often enters implicitly: the skill-conditioned behavior is treated as a better source of supervision.
\textbf{\textit{Can skill-conditioned context be safely treated as a privileged teacher during agent training?}}

This assumption is appealing but fragile.
Retrieved skills may be irrelevant, stale, overly generic, or locally mismatched to the current state, and the agent may fail to ground even useful skills in the present observation.
Consequently, a skill-conditioned prompt can help in one state but mislead the same policy in another.
Prior methods improve skill construction, curation, or skill-conditioned self-distillation~\citep{ouyang2026skilloslearningskillcuration,wang2026skillsdskillconditionedselfdistillation,lu2026selfdistilledagenticreinforcementlearning}, but the supervision direction often remains fixed from the skill-conditioned view to the no-skill view.
Our diagnostics in Section~\ref{sec:observations} challenge this fixed-teacher view: the skill-conditioned branch is not consistently better than the no-skill branch, and making skill-conditioned rollouts on-policy does not remove this ambiguity.
\textbf{\textit{When skill and no-skill views disagree at the same state, which view should teach the other?}}

To address this, we propose \textsc{UCOB} (learning to \textbf{U}tilize and evolve agentic skills via \textbf{C}redit-aware \textbf{O}n-policy \textbf{B}idirectional self-distillation), which treats the skill-conditioned prompt and the no-skill prompt as two on-policy context views of the same online model, rather than as a fixed teacher-student pair.
The phrase \emph{utilize and evolve} reflects a sequential loop: \textsc{UCOB} first decides whether retrieved skills should be trusted or corrected, then uses the resulting credit evidence to update the skill memory and skill writer.
During training, \textsc{UCOB} compares skill/no-skill rollouts within the same task-anchor group and lets the higher-return view teach the other.
Thus, useful skill-conditioned behavior is internalized, while misleading skill-induced behavior is corrected by the no-skill view.
The same local credit also updates the dual-granularity skill memory, guides utility-aware retrieval, and trains the reflection-based skill writer, closing the loop between policy learning and skill evolution.
Figure~\ref{fig:overall} summarizes the resulting framework and its empirical behavior across training and evaluation.

Experiments on agentic tasks, including \textsc{ALFWorld}, \textsc{WebShop}, and \textsc{Search-QA}, show that \textsc{UCOB} improves over skill-free RL, skill-augmented baselines, and self-distillation methods across model scales.
Ablations verify the contribution of each module and design choice, while further analyses examine teacher routing, memory evolution, continual adaptation across sequential environments, and training cost.
The main contributions of this work are as follows:
\begin{tcolorbox}[
    enhanced,
    colback=black!2,
    colframe=black!15,
    boxrule=0.35pt,
    arc=1.5mm,
    sharp corners=southwest,
    sharp corners=southeast,
    left=1.2mm,right=1.2mm,top=1.2mm,bottom=1.2mm
]
\begin{enumerate}[leftmargin=2em, itemsep=2pt, topsep=2pt]
    \item We motivate and propose \textbf{Credit-Aware Bidirectional Self-Distillation (CBSD)}, which selects on-policy teacher directions from same-task, same-anchor-state returns and is supported by a local policy-improvement analysis of credited bidirectional updates.
    \item We develop \textbf{\textsc{UCOB}}, a unified framework that utilizes skills via CBSD and evolves them with dual-granularity skill memory, utility-aware retrieval, and reflection self-training.
    \item We validate \textsc{UCOB} across model scales, outperforming SOTA baselines such as SDAR and Skill1 by up to \textbf{23.5} and \textbf{18.0} absolute success-rate gains on \textsc{ALFWorld} and \textsc{WebShop}; analyses further confirm its module and design effectiveness.
\end{enumerate}
\end{tcolorbox}


\section{Observations}
\label{sec:observations}

We test whether skill-conditioned context can serve as a reliable privileged teacher through three diagnostics: \textbf{(i)} whether it consistently outperforms the no-skill view, \textbf{(ii)} whether making it on-policy resolves teacher ambiguity, and \textbf{(iii)} which local evidence should choose the distillation direction when the two views disagree.
Figure~\ref{fig:obs-settings} summarizes the diagnostic protocols: panel (a) covers the fixed-direction skill/no-skill setups for testing teacher reliability and on-policy exposure, while panel (b) illustrates the anchor-state return comparison for selecting the local teacher direction.

\begin{figure*}[t]
    \centering
    \begin{subfigure}[t]{0.58\textwidth}
        \centering
        \includegraphics[width=\linewidth]{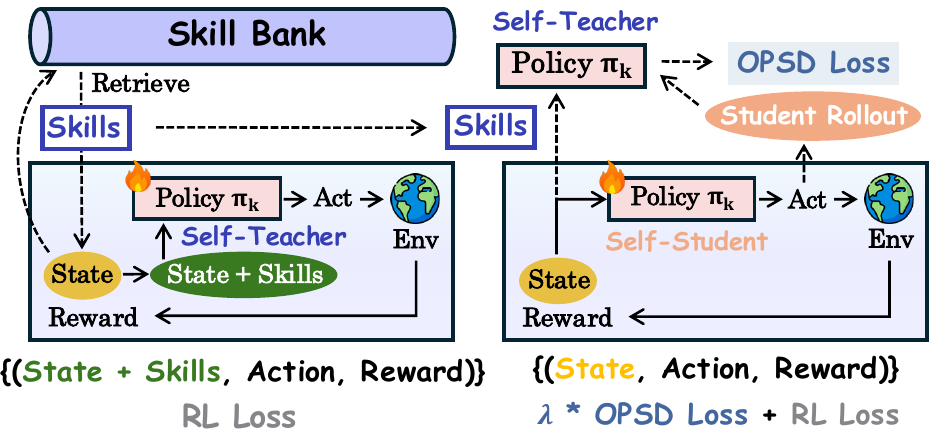}
        \caption{Fixed-direction skill/no-skill protocol}
        \label{fig:obs-settings-a}
    \end{subfigure}
    \hfill
    \begin{subfigure}[t]{0.39\textwidth}
        \centering
        \includegraphics[width=\linewidth]{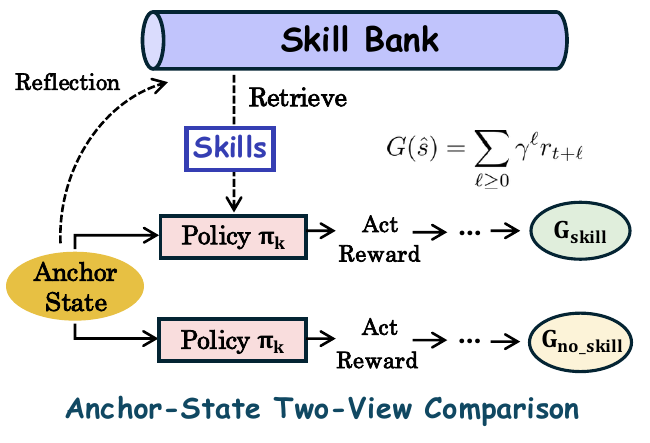}
        \caption{Anchor-state return-gap diagnostic}
        \label{fig:obs-settings-b}
    \end{subfigure}
    \caption{Unified schematic of the observation-study protocols. (a) Fixed-direction skill/no-skill protocol for the no-skill-only and dual-rollout settings. (b) Anchor-state return-gap diagnostic. Detailed settings and rollout details are given in Appendix~\ref{app:obs-settings}.}
    \label{fig:obs-settings}
    \vspace{-1em}
\end{figure*}

\paragraph{Skill-conditioned teachers are unreliable and not self-correcting.}
\begin{wrapfigure}{r}{0.4\textwidth}
    \vspace{-2em}
    \centering
    \includegraphics[width=0.39\textwidth]{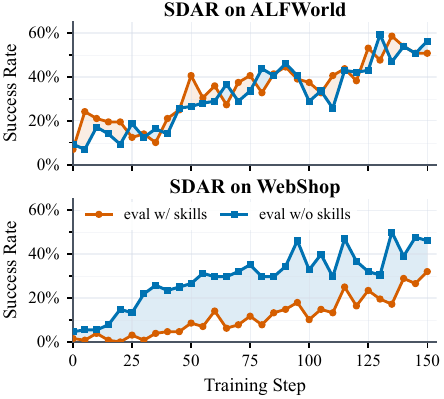}
    \vspace{-0.5em}
    \caption{SDAR Evaluation w/ and w/o skills during training on Qwen3-1.7B.}
    \label{fig:sdar-dual-acc}
    \vspace{-0.5em}
\end{wrapfigure}
We first revisit SDAR~\citep{lu2026selfdistilledagenticreinforcementlearning}, an asymmetric self-distillation setup where rollouts use the \textit{no-skill} prompt while a \textit{skill-conditioned} prompt serves as the privileged teacher. 
This design assumes that, for states induced by the no-skill rollout, the skill-conditioned view provides a more reliable distillation target than the corresponding no-skill view. 
Figure~\ref{fig:sdar-dual-acc} challenges this assumption on \textsc{ALFWorld} and \textsc{WebShop}: the with-skill evaluation does not consistently dominate the no-skill evaluation, and on \textsc{WebShop}, it remains weaker throughout training. 
Thus, \textbf{\textit{adding a skill context does not automatically yield a reliable teacher}}. 
Since the privileged view is not directly corrected by return feedback under its own context distribution, a weak teacher is also \textbf{\textit{not self-correcting}}; distillation may continue to transfer supervision from unverified conditioned behavior.

\paragraph{Rollouts with skills mitigate exposure mismatch but not teacher ambiguity.}
\begin{wrapfigure}{r}{0.4\textwidth}
    \vspace{-4em}
    \centering
    \includegraphics[width=0.39\textwidth]{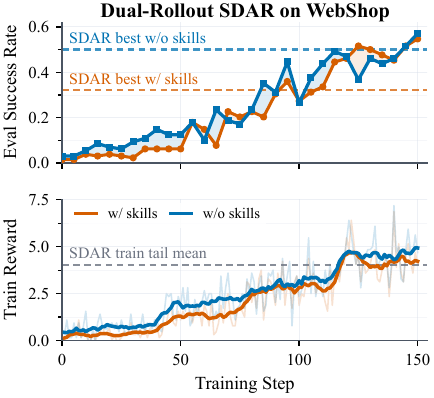}
    \vspace{-0.5em}
    \caption{Dual-rollout SDAR evaluation and training on \textsc{WebShop}.}
    \label{fig:dsdar-dual-acc}
    \vspace{-1em}
\end{wrapfigure}
A natural remedy is to place skills into the rollout itself, so the skill-conditioned view is optimized through environment interaction rather than only serving as a teacher outside the rollout path. 
We therefore test a dual-rollout fixed-direction variant: each training batch samples half of the trajectories with skills and half without skills, following the training design of Skill-SD~\citep{wang2026skillsdskillconditionedselfdistillation}, while keeping the SDAR-style skill-to-no-skill distillation direction.
Both context views therefore receive on-policy RL updates, reducing exposure mismatch. 
This dual-rollout setting is summarized in Figure~\ref{fig:obs-settings}(a).
However, distillation remains one-way: only the no-skill branch receives the extra distillation loss, while the skill-conditioned branch is updated by RL alone. 
Figure~\ref{fig:dsdar-dual-acc} shows that this stronger baseline narrows the gap between views, but still does not make the skill-conditioned view consistently better. 
On \textsc{WebShop}, skill-conditioned training rollouts also have lower average reward than no-skill rollouts. 
Thus, \textbf{\textit{making the skill-conditioned view on-policy does not make it universally authoritative}}.
The question becomes whether current skill-induced behavior should be \textbf{\textit{trusted and internalized when useful, or corrected by the no-skill view when misleading}}.

\begin{wrapfigure}{r}{0.4\textwidth}
    \vspace{-1.0em}
    \centering
    \includegraphics[width=0.39\textwidth]{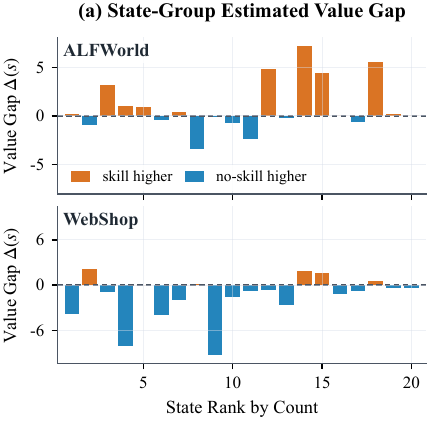}\\[0.4em]
    \includegraphics[width=0.39\textwidth]{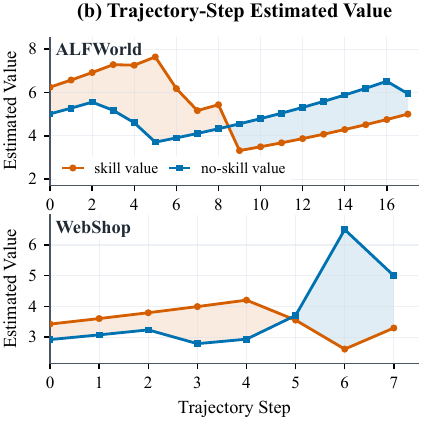}
    \vspace{-0.5em}
    \caption{State-group and trajectory-step diagnostics for two-view rollouts.}
    \label{fig:fig3-motivation}
    \vspace{-2em}
\end{wrapfigure}

\paragraph{Teacher direction is locally value-dependent.}
To decide whether skill-induced behavior should be trusted at a state, we compare skill-conditioned and no-skill views within the same task and anchor state. 
Following GiGPO~\citep{feng2025groupingrouppolicyoptimizationllm}, we group rollouts by anchor state and estimate each view by average return-to-go. 
For a rollout record reaching state $s_t$, we use $G(s_t)=\sum_{\ell\ge 0}\gamma^\ell r_{t+\ell}$, average it within each task-anchor group, and define the same-state value gap as $\Delta(\tilde{s})=\bar G_{+}(\tilde{s})-\bar G_{0}(\tilde{s})$, where $+$ and $0$ denote the two views. 
A positive gap indicates useful skill context to internalize, whereas a negative gap marks skill-conditioned behavior that is locally less reliable and should be corrected by the no-skill view. 
Figure~\ref{fig:obs-settings}(b) illustrates this anchor-state comparison.
Using a Qwen3-1.7B checkpoint trained with dual-rollout SDAR, Figure~\ref{fig:fig3-motivation}(a) ranks frequent state groups and plots the skill-minus-no-skill value gap. 
The gaps are mixed-sign in both environments, showing that \textbf{\textit{teacher quality is not one-sided}}; teacher direction should depend on the local value gap within each state group. 
Figure~\ref{fig:fig3-motivation}(b) further shows that the better branch can switch across decision steps within a sampled trajectory. 
Thus, even when a trajectory is good overall, not every step provides reliable supervision; a trajectory-level teacher may give the right global direction but the wrong local update. 
This makes \textbf{\textit{trajectory summaries too coarse for deciding each update}}.

\paragraph{From observations to \textsc{UCOB}.}
The above observations turn state grouping into credit-aware cross-context supervision. 
For each task and anchor state, matched skill/no-skill rollouts form a credited pair: the higher-value branch gives behavior to distill into the opposite context, while the lower-value branch reveals which context should not be trusted as teacher. 
The same credited pairs also provide a more reliable source of state skills than trajectory-level summarization, unlike prior dual-granularity methods such as D2Skill~\citep{tu2026dynamicdualgranularityskillbank}, which summarize state skills from trajectories or trajectory fragments rather than \textbf{\textit{anchor-state value comparisons}}.
Therefore, we introduce \textbf{\textsc{UCOB}}, a unified framework that operationalizes credited state comparisons for skill learning: it generates and retrieves both task- and state-level skills, uses \textbf{\textit{local value gaps}} to utilize skills through \textbf{\textit{credit-aware bidirectional self-distillation}}, and evolves both the answer policy and the skill generator. 
These observations also identify missing axes in prior skill-conditioned agent methods; Table~\ref{tab:method-positioning} positions \textsc{UCOB} against representative skill-augmented agentic RL and self-distillation methods.

\begin{table}[h]
    \centering
    \caption{Comparison of representative skill-conditioned agent methods. Task-level skills are assumed throughout this family and are therefore omitted; the skill column indicates whether a method additionally maintains state-level skills. The blue triangle for SDAR denotes partial credit-aware supervision: it performs self-distillation with a token-level gate, but does not explicitly suppress harmful teachers by redirecting supervision to the no-skill branch. 
    These axes are later tested through component ablations and analysis in Sections~\ref{sec:ablation}--\ref{sec:analysis}.}
    \label{tab:method-positioning}
    \vspace{-0.5em}
    \scriptsize
    \setlength{\tabcolsep}{2.5pt}
    \resizebox{\textwidth}{!}{
    \begin{tabular}{lccccccc}
        \toprule
        \makecell[l]{\textbf{Method}} & \makecell[c]{\textbf{Date}\\\textbf{(YY.M)}} & \makecell[c]{\textbf{State-Level}\\\textbf{Skills}} & \makecell[c]{\textbf{Credit-Aware}\\\textbf{Supervision}} & \makecell[c]{\textbf{Skill-Free}\\\textbf{Inference}} & \makecell[c]{\textbf{Evolving}\\\textbf{Skill Memory}} & \makecell[c]{\textbf{Evolving}\\\textbf{Skill Policy}} \\
        \midrule
        SkillRL~\citep{xia2026skillrlevolvingagentsrecursive} & 26.2 & \xmark & \xmark & \xmark & \cmark & \xmark \\
        D2Skill~\citep{tu2026dynamicdualgranularityskillbank} & 26.3 & \cmark & \xmark & \xmark & \cmark & \xmark \\
        RetroAgent~\citep{zhang2026retroagentsolvingevolvingretrospective} & 26.3 & \xmark & \xmark & \xmark & \cmark & \cmark \\
        Skill-SD~\citep{wang2026skillsdskillconditionedselfdistillation} & 26.4 & \xmark & \xmark & \cmark & \xmark & \xmark \\
        Skill1~\citep{shi2026skill1unifiedevolutionskillaugmented} & 26.5 & \xmark & \xmark & \xmark & \cmark & \cmark \\
        SDAR~\citep{lu2026selfdistilledagenticreinforcementlearning} & 26.5 & \xmark & \pmark & \cmark & \xmark & \xmark \\
        StepOPSD~\citep{zhang2026stepopsdstepawareonlinepreference} & 26.5 & \xmark & \xmark & \cmark & \xmark & \xmark \\
        SkillC~\citep{lin2026skillclearningautonomousskill} & 26.5 & \xmark & \cmark & \cmark & \cmark & \xmark \\
        SAPO~\citep{zhang2026coevolvingskillgenerationpolicy} & 26.6 & \xmark & \xmark & \xmark & \cmark & \cmark \\
        \textbf{\textsc{UCOB (Ours)}} & \textbf{\textsc{26.6}}& \cmark & \cmark & \cmark & \cmark & \cmark \\
        \bottomrule
    \end{tabular}}
    \vspace{-1em}
\end{table}


\section{Problem Setup and Preliminaries}
\label{sec:preliminaries}

\paragraph{Problem setup.}
We consider reinforcement learning for multi-turn language agents with retrieved skills.
Each task instance is indexed by $u$.
At turn $t$, the agent observes a state $s_t$ containing the task description, current observation, and interaction history.
The policy $\pi_\theta$ generates a textual response $y_t=(y_{t,1},\ldots,y_{t,L_t})$, from which an executable action is parsed and sent to the environment.
Let $\gamma\in[0,1)$ denote the discount factor, and write $\tilde{s}=\operatorname{anchor}(s)$ for the canonical anchor-state abstraction used in local comparisons.
The environment returns reward $r_t$ and the next state; we denote the return-to-go by $G_t=\sum_{\ell\ge0}\gamma^\ell r_{t+\ell}$.
A skill-conditioned agent retrieves textual skills $M(s_t)$ and forms a skill-conditioned prompt $P_+(s_t)=\operatorname{Prompt}(s_t,M(s_t))$, while the no-skill prompt is $P_0(s_t)=\operatorname{Prompt}(s_t,\emptyset)$.
These prompts induce two context views of the same policy, $\pi_+(\cdot\mid s_t)=\pi_\theta(\cdot\mid P_+(s_t))$ and $\pi_0(\cdot\mid s_t)=\pi_\theta(\cdot\mid P_0(s_t))$.

\paragraph{Skill-conditioned OPSD.}
On-policy self-distillation (OPSD) trains on student-sampled rollouts and minimizes a per-token divergence between teacher and student distributions along these on-policy prefixes~\citep{zhao2026selfdistilledreasoner}.
For a sampled response $y$, teacher prompt $P_{\mathrm{tea}}$, and student prompt $P_{\mathrm{stu}}$, we write a generic OPSD loss as
\begin{equation}
    \mathcal L_{\mathrm{OPSD}}
    =\frac{1}{Z_{\mathrm{OPSD}}}\sum_j m_j
    D_{\mathrm{KL}}\!\left(
    \operatorname{sg}\!\left[\pi_\theta(\cdot\mid P_{\mathrm{tea}},y_{<j})\right]\,\Vert\,
    \pi_\theta(\cdot\mid P_{\mathrm{stu}},y_{<j})
    \right),
    \end{equation}
where $m_j$ masks valid response tokens, $Z_{\mathrm{OPSD}}=\sum_j m_j$ normalizes over valid positions, and $\operatorname{sg}[\cdot]$ denotes stop-gradient on the teacher distribution.
Skill-conditioned variants such as Skill-SD and SDAR typically set the skill-conditioned view as the teacher, i.e., $P_{\mathrm{tea}}=P_+$ and $P_{\mathrm{stu}}=P_0$~\citep{wang2026skillsdskillconditionedselfdistillation,lu2026selfdistilledagenticreinforcementlearning}.
This fixed direction provides dense supervision but does not verify whether the skill-conditioned view is locally better.

\paragraph{On-policy agentic RL and credit assignment.}
Our online RL backbone follows GiGPO~\citep{feng2025groupingrouppolicyoptimizationllm}, which extends group-based RL with step-level credit for multi-turn agents.
For the rollout group of task $u$, let $\tau_i=\{(s_{i,t},y_{i,t},r_{i,t})\}_{t=1}^{T_i}$ be a rollout from $\pi_{\theta_{\mathrm{old}}}$, $R_i=\sum_t r_{i,t}$ its episode return, and $G_{i,t}$ the return-to-go at turn $t$.
GiGPO combines episode-level credit with anchor-state credit:
\begin{equation}
    A_i^E=\frac{R_i-\mu_u^E}{F_u^E},\quad
    A_{i,t}^S=\frac{G_{i,t}-\mu_{u,\tilde{s}}^S}{F_{u,\tilde{s}}^S},\quad
    A_{i,t}=A_i^E+\omega A_{i,t}^S,\quad
    \tilde{s}=\operatorname{anchor}(s_{i,t}).
\end{equation}
Here $\mu$ and $F$ are the mean and normalizer within the corresponding episode or anchor-state group, and $\omega$ weights step-level credit.
Broadcasting $A_{i,t}$ to valid response tokens yields the clipped on-policy loss
\begin{equation}
    \mathcal L_{\mathrm{RL}}
    =-\frac{1}{Z_{\mathrm{RL}}}\sum_{i,t,j}m_{i,t,j}
    \min\!\left(
    \varrho_{i,t,j}A_{i,t},
    \operatorname{clip}(\varrho_{i,t,j},1-\epsilon_{\mathrm{clip}},1+\epsilon_{\mathrm{clip}})A_{i,t}
    \right),
    \end{equation}
where $m_{i,t,j}$ is the valid-token mask and $Z_{\mathrm{RL}}=\sum_{i,t,j}m_{i,t,j}$.
Let $c_{i,t}\in\{+,0\}$ denote whether the response was sampled under the skill-conditioned or no-skill prompt.
The token-level context is $h_{i,t,j}=(P_{c_{i,t}}(s_{i,t}),y_{i,t,<j})$, and the importance ratio is $\varrho_{i,t,j}=\pi_\theta(y_{i,t,j}\mid h_{i,t,j})/\pi_{\theta_{\mathrm{old}}}(y_{i,t,j}\mid h_{i,t,j})$.
Here $\epsilon_{\mathrm{clip}}$ is the PPO-style clipping range.
We additionally use the reference-policy regularizer $\mathcal L_{\mathrm{KL}}=\mathbb E[D_{\mathrm{KL}}(\pi_\theta\|\pi_{\mathrm{ref}})]$.
Later, \textsc{UCOB} reuses the same anchor-state groups as local credited records for cross-context comparisons; the exact record construction is defined in Section~\ref{subsec:mixed_rollout_grouping}.

\section{Method: \textsc{UCOB}}
\label{sec:method}

\textsc{UCOB} treats retrieved skills as candidate context rather than privileged labels.
Figure~\ref{fig:overall}(a) gives a compact sketch of this four-stage loop, and Figure~\ref{fig:ucob-main} expands it into the full method.
\textsc{UCOB} forms a closed loop that first utilizes retrieved skills as candidate context, then evolves the skill memory from credited interaction evidence.
At its core, \textsc{UCOB} compares skill-conditioned and no-skill views at the same task-anchor state, distills from the higher-return view, and feeds the resulting credit back into policy and skill updates.
Appendix~\ref{app:full-algorithm} provides the full procedure in Algorithm~\ref{alg:ucob-full}.

\begin{figure*}[t]
    \centering
    \includegraphics[width=\textwidth]{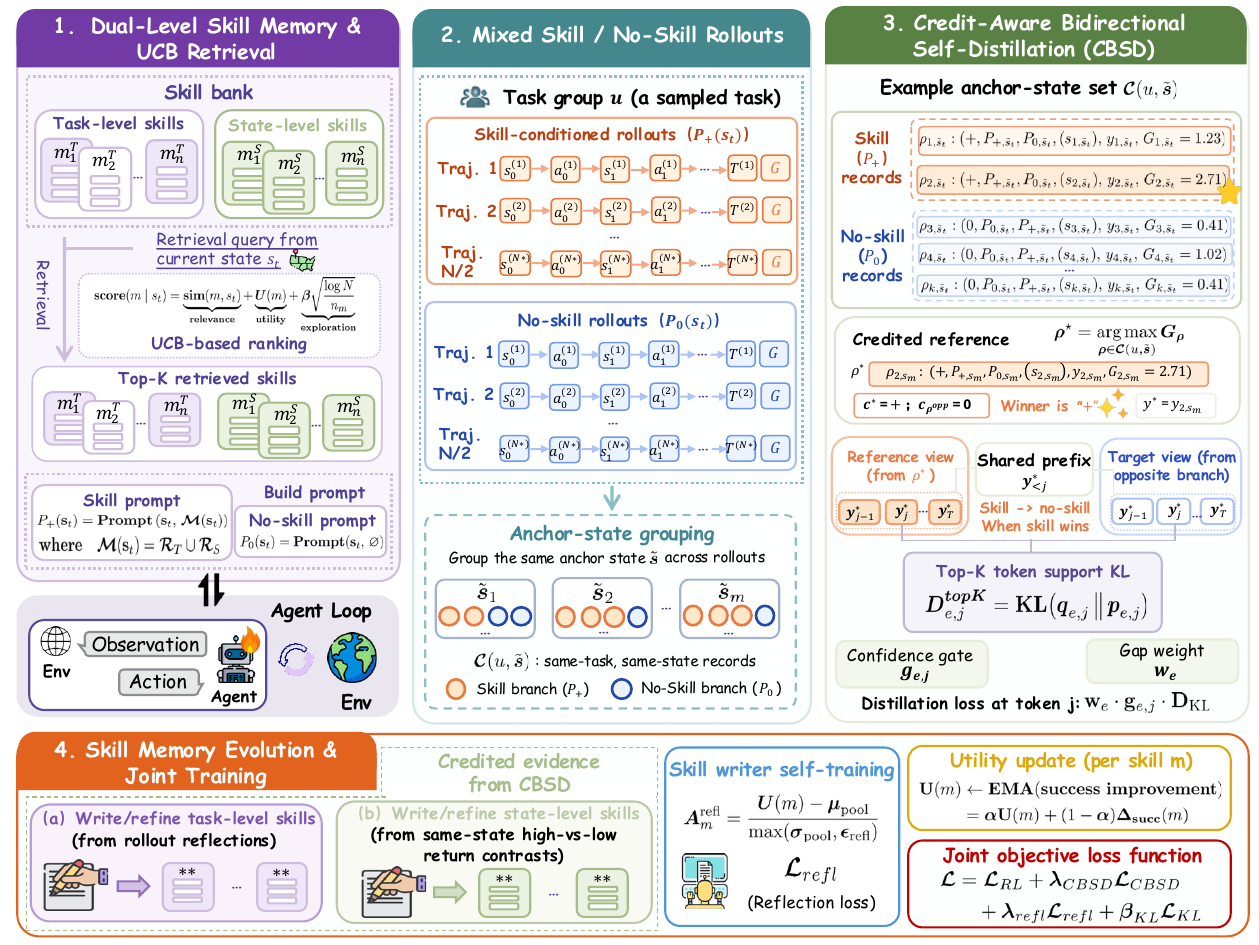}
    \caption{Overview of \textsc{UCOB} as a \textbf{four-stage training loop}: (1) retrieve dual-granularity skills to build skill/no-skill views, (2) collect mixed rollouts and group anchor states, (3) route local supervision via credit-aware bidirectional self-distillation, and (4) jointly update the skill memory and policy. The updated memory and policy are reused in the next round.}
    \label{fig:ucob-main}
    \vspace{-1em}
\end{figure*}

\subsection{Step 1: Dual-Granularity Skill Retrieval}
\label{subsec:dual_memory_rollout}


\paragraph{Utility-aware UCB skill retrieval.}
\textsc{UCOB} maintains a dual-granularity skill memory, consisting of task-level and state-level skill pools, and expands it online from rollout reflections.
Task-level skills capture reusable episode-level strategies, while state-level skills capture local decision rules grounded in credited same-state contrasts; the writing rule is detailed in Section~\ref{subsec:skill_evolution_objective}.
Each skill item $m$ contains a textual principle, an applicability condition, a retrieval key $k_m$, a utility score $U(m)$, and a usage count $n_m$.
For each pool $p\in\{\mathrm{task},\mathrm{state}\}$, \textsc{UCOB} ranks candidate skills with an Upper Confidence Bound (UCB) score:
\begin{equation}
    R_p(m;s_t)=
    \alpha\,\operatorname{sim}\!\left(q_p(s_t),k_m\right)
    +(1-\alpha)\left(U(m)+\eta\sqrt{\frac{\log (N_p+1)}{n_m+1}}\right),
    \label{eq:ucb-retrieval}
\end{equation}
where $q_p(s_t)$ is the task- or state-level retrieval query, $N_p$ is the total usage count of pool $p$, and $\eta$ controls UCB exploration.
The retrieved memory sets are
\begin{equation}
    M_p(s_t)=\operatorname{TopK}^{K_p^{\mathrm{mem}}}_{m\in\mathcal M_p} R_p(m;s_t),
    \qquad p\in\{\mathrm{task},\mathrm{state}\}.
    \label{eq:retrieved-memory}
\end{equation}
Here $K_p^{\mathrm{mem}}$ is the skill retrieval budget for pool $p$.
The union $M_{\mathrm{task}}(s_t)\cup M_{\mathrm{state}}(s_t)$ is inserted into the skill-conditioned prompt $P_{+}(s_t)$, while the no-skill prompt $P_0(s_t)$ receives the same state information without retrieved skills.

\subsection{Step 2: Mixed Rollouts and Anchor-State Grouping}
\label{subsec:mixed_rollout_grouping}

\paragraph{Mixed skill/no-skill rollouts.}
For each task group, \textsc{UCOB} samples mixed rollouts from the current policy under the two prompt views.
Half of the trajectories act with the skill-conditioned prompt $P_{+}(s_t)$, and the remaining trajectories act with the no-skill prompt $P_0(s_t)$.
Both branches interact with the environment and are optimized by the online RL loss.
This differs from fixed privileged-teacher setups~\citep{lu2026selfdistilledagenticreinforcementlearning}, where the skill-conditioned view may be used only outside the rollout path.
No branch is fixed as the teacher during rollout; teacher selection is deferred to same-anchor-state return comparisons.

\paragraph{Anchor-state grouping.}
Following the anchor-state grouping mechanism of GiGPO~\citep{feng2025groupingrouppolicyoptimizationllm}, \textsc{UCOB} retroactively groups repeated states inside each task group.
Let $\mathcal A_u=\{\operatorname{anchor}(s_{i,t})\}_{i,t}$ be the set of distinct anchor states visited by rollouts of task group $u$.
For each anchor $\tilde{s}\in\mathcal A_u$, we construct
\begin{equation}
    \mathcal C(u,\tilde{s})=
    \{\rho_{i,t}\mid \rho_{i,t}\text{ is from task }u,\ 
    \operatorname{anchor}(s_{i,t})=\tilde{s}\}.
    \label{eq:anchor-group}
\end{equation}
Each record is represented as
\begin{equation}
    \rho_{i,t}=\big(s_{i,t}, c_{i,t}, P_{c_{i,t}}(s_{i,t}), P_{\bar c_{i,t}}(s_{i,t}), y_{i,t}, G_{i,t}\big),
\end{equation}
where $c_{i,t}\in\{+,0\}$ is the rollout branch, $\bar c_{i,t}$ denotes the opposite branch, $P_{c_{i,t}}$ is the acting prompt, $P_{\bar c_{i,t}}$ is the opposite prompt, $y_{i,t}$ is the sampled response, and $G_{i,t}$ is the return-to-go.
The opposite prompt is not used to generate the action; it allows a response sampled under one context view to supervise the other view in Section~\ref{subsec:bidirectional_opd}.
Thus, CBSD operates on same-task, same-anchor-state comparison sets.

\subsection{Step 3: Credit-Aware Bidirectional Self-Distillation}
\label{subsec:bidirectional_opd}

CBSD is the core credit-assignment mechanism of \textsc{UCOB}.
For each same-task, same-anchor-state group, it uses return evidence to decide which context view should teach the other, then distills the higher-return response distribution to the opposite view.
Here \emph{bidirectional} means credit-routed teacher selection: skill to no-skill when retrieved skills help, and no-skill to skill when the skill context is locally harmful.

\paragraph{Credited direction selection.}
For each anchor-state set $\mathcal C(u,\tilde{s})$, \textsc{UCOB} first selects the highest-return record as the credited reference:
\begin{equation}
    \rho^\star=\arg\max_{\rho\in\mathcal C(u,\tilde{s})}G_\rho,
    \qquad
    c^\star=c_{\rho^\star},
    \qquad
    y^\star=y_{\rho^\star}.
    \label{eq:credited-reference}
\end{equation}
Here $\rho^\star$ is the local winner, $c^\star$ is its prompt branch, and $y^\star$ is the response to be distilled.
For every record $\rho^{\mathrm{opp}}$ from the opposite branch, i.e., $c_{\rho^{\mathrm{opp}}}=\bar c^\star$, we form a candidate pair $e=(\rho^\star,\rho^{\mathrm{opp}})$ with return gap $\Delta_e=G_{\rho^\star}-G_{\rho^{\mathrm{opp}}}$.
We keep the pair in $\mathcal P_{\mathrm{CBSD}}$ only if $\Delta_e>\epsilon_{\mathrm{CBSD}}$.
For each accepted pair, $P^{\mathrm{ref}}_e$ is the acting prompt stored in $\rho^\star$, and $P^{\mathrm{tgt}}_e$ is a target prompt from the opposite context view within the same anchor-state set.
If $c^\star=+$, the skill-conditioned view teaches the no-skill view; if $c^\star=0$, the no-skill view teaches the skill-conditioned view.
Thus, the teacher direction is selected by same-task, same-anchor-state return evidence rather than by the presence of retrieved skills.

\paragraph{Top-$K$ token-support matching.}
For each accepted pair, \textsc{UCOB} applies response-level token top-$K$ OPD, inspired by teacher top-$K$ local support matching~\citep{fu2026revisitingonpolicydistillation}.
At response position $j$, both views are evaluated on the same credited prefix $y^\star_{<j}$, but under different prompts:
\[
    h^{\mathrm{ref}}_{e,j}=(P^{\mathrm{ref}}_e,y^\star_{<j}),
    \qquad
    h^{\mathrm{tgt}}_{e,j}=(P^{\mathrm{tgt}}_e,y^\star_{<j}).
\]
The credited reference view defines the local token support to be matched:
\begin{equation}
    S_{e,j}=\operatorname{TopK}^{K_{\mathrm{tok}}}\!\left(\pi_\theta(\cdot\mid h^{\mathrm{ref}}_{e,j})\right).
\end{equation}
Here $K_{\mathrm{tok}}$ is the token-support size, distinct from the memory retrieval budget $K_p^{\mathrm{mem}}$.
On $S_{e,j}$, we renormalize the reference and target distributions:
\begin{equation}
    q_{e,j}(v)=\operatorname{sg}\!\left[
    \frac{\pi_\theta(v\mid h^{\mathrm{ref}}_{e,j})}
    {\sum_{z\in S_{e,j}}\pi_\theta(z\mid h^{\mathrm{ref}}_{e,j})}
    \right],
    \qquad
    p_{e,j}(v)=\frac{\pi_\theta(v\mid h^{\mathrm{tgt}}_{e,j})}
    {\sum_{z\in S_{e,j}}\pi_\theta(z\mid h^{\mathrm{tgt}}_{e,j})},
    \quad v\in S_{e,j}.
\end{equation}
The token-level matching term is $D_{e,j}^{\mathrm{top}K}=\mathrm{KL}(q_{e,j}\|p_{e,j})$.
Compared with sampled-token supervision, this transfers the reference view's local distributional preference to the target view.

\paragraph{CBSD objective.}
Following the token-level gating idea in SDAR~\citep{lu2026selfdistilledagenticreinforcementlearning}, we use a confidence gate to suppress positions where the credited reference is not more reliable than the target:
\begin{equation}
    g_{e,j}=\sigma\!\left(\beta_{\mathrm{gate}}\left[
    \log\pi_\theta(y_j^\star\mid h^{\mathrm{ref}}_{e,j})
    -\log\pi_\theta(y_j^\star\mid h^{\mathrm{tgt}}_{e,j})
    \right]\right).
\end{equation}
Here $\sigma(\cdot)$ is the sigmoid function and $\beta_{\mathrm{gate}}$ controls the sharpness of the gate.
The reference branch can be either skill-conditioned or no-skill.
The resulting credit-aware bidirectional self-distillation objective is
\begin{equation}
    \mathcal L_{\mathrm{CBSD}}=
    \frac{1}{Z_{\mathrm{CBSD}}}\sum_{e\in\mathcal P_{\mathrm{CBSD}}}\sum_j
    \chi_{e,j}D_{e,j}^{\mathrm{top}K},
    \label{eq:cbsd-objective}
\end{equation}
where $\chi_{e,j}=m_{e,j}w_e g_{e,j}$, $w_e=\operatorname{clip}(\Delta_e/\tau_{\mathrm{CBSD}},0,w_{\max})$, and $Z_{\mathrm{CBSD}}=\sum_{e,j}m_{e,j}$.
Here $m_{e,j}$ masks valid response tokens, $w_e$ scales supervision by the same-task, same-anchor-state return gap, and $g_{e,j}$ applies the confidence gate.
In practice, $\mathcal L_{\mathrm{CBSD}}$ provides auxiliary cross-context supervision alongside the online RL loss.
Section~\ref{sec:theory} gives a local policy-improvement interpretation of this credited direction selection, with detailed proofs in Appendix~\ref{app:theory-cbsd}.

\subsection{Step 4: Skill Memory Evolution and Joint Training}
\label{subsec:skill_evolution_objective}

\textsc{UCOB} closes the loop by writing skills, updating their utilities, and training the reflection-based skill writer to improve future skill generation.

\paragraph{Skill writing and utility update.}
The rollout records used for CBSD also provide evidence for writing and scoring skills.
\textsc{UCOB} writes task-level skills by reflecting on rollout groups, and writes state-level skills by contrasting higher- and lower-return records within the same anchor-state set $\mathcal C(u,\tilde{s})$.
A state-level skill records when the state occurs, which behavior worked better, and which behavior should be avoided; hence it is grounded in a local value contrast rather than a trajectory-level summary.

To score a used skill $m$, let $b_u$ be the mean no-skill success of task group $u$, and let $\delta_m$ be the success improvement of rollouts using $m$ over $b_u$.
We compute this credit separately for task- and state-level skill pools.
\textsc{UCOB} updates the skill utility by an exponential moving average (EMA) with rate $\beta_U$:
\begin{equation}
    U(m)\leftarrow (1-\beta_U)U(m)+\beta_U\delta_m.
    \label{eq:utility-update}
\end{equation}
The updated utilities feed back into the UCB retrieval score in Section~\ref{subsec:dual_memory_rollout}, so skills that repeatedly help under on-policy rollouts become easier to retrieve later.

\paragraph{Skill-writer self-training.}
\textsc{UCOB} further improves the skill writer by replaying its own reflection prompt-response pairs with utility-derived advantages.
For a stored skill $m$, we compute a pool-normalized reflection advantage $A_m^{\mathrm{refl}}=(U(m)-\mu_{\mathrm{pool}})/\max(\sigma_{\mathrm{pool}},\epsilon_{\mathrm{refl}})$, with statistics computed separately for task-level and state-level pools.
Here $\mu_{\mathrm{pool}}$ and $\sigma_{\mathrm{pool}}$ are the mean and standard deviation of utilities in the corresponding pool, and $\epsilon_{\mathrm{refl}}$ is a small numerical floor.
Given the reflection prompt $x_m$ and response $z_m$, we optimize a clipped policy-gradient loss:
\begin{equation}
    \mathcal L_{\mathrm{refl}}
    =
    \frac{1}{Z_{\mathrm{refl}}}\sum_{m,j} \ell_{m,j}^{\mathrm{refl}}
    \max\!\left(
    -\varrho_{m,j}^{\mathrm{refl}}A_m^{\mathrm{refl}},
    -\operatorname{clip}(\varrho_{m,j}^{\mathrm{refl}},1-\epsilon_{\mathrm{clip}},1+\epsilon_{\mathrm{clip}})A_m^{\mathrm{refl}}
    \right),
    \label{eq:reflection-loss}
\end{equation}
where $\ell_{m,j}^{\mathrm{refl}}$ masks valid reflection-response tokens, $\varrho_{m,j}^{\mathrm{refl}}=\pi_\theta(z_{m,j}\mid x_m,z_{m,<j})/\pi_{\theta_{\mathrm{old}}}(z_{m,j}\mid x_m,z_{m,<j})$, and $Z_{\mathrm{refl}}=\sum_{m,j}\ell_{m,j}^{\mathrm{refl}}$.
Thus, high-utility memories increase the likelihood of their reflection responses, while low-utility memories are downweighted or suppressed.

\paragraph{Overall objective.}
\textsc{UCOB} minimizes the combined training loss:
\begin{equation}
    \mathcal L=
    \mathcal L_{\mathrm{RL}}
    +\lambda_{\mathrm{CBSD}}\mathcal L_{\mathrm{CBSD}}
    +\lambda_{\mathrm{refl}}\mathcal L_{\mathrm{refl}}
    +\beta_{\mathrm{KL}}\mathcal L_{\mathrm{KL}}.
    \label{eq:ucob-objective}
\end{equation}
Here $\lambda_{\mathrm{CBSD}}$, $\lambda_{\mathrm{refl}}$, and $\beta_{\mathrm{KL}}$ balance the distillation, reflection, and KL regularization terms.
The base losses $\mathcal L_{\mathrm{RL}}$ and $\mathcal L_{\mathrm{KL}}$ are defined in Section~\ref{sec:preliminaries}, while $\mathcal L_{\mathrm{CBSD}}$ and $\mathcal L_{\mathrm{refl}}$ are introduced above.
This objective couples online policy improvement with credited cross-context distillation and utility-guided skill evolution.

\section{Theoretical Perspective}
\label{sec:theory}

This section supports the CBSD contribution by explaining why same-anchor-state return gaps can define a useful teacher direction.
We give a local, one-update interpretation of CBSD, not a global convergence claim for the full nonstationary \textsc{UCOB} system; full proofs and a conditional contraction result are in Appendix~\ref{app:theory-cbsd}.
Fix an old policy $\pi$, one task-anchor context $x=(u,\tilde{s})$, and the old-policy continuation after a response $y$.
For branch $c\in\{+,0\}$, let $p_c(y\mid x)$ denote the response distribution of the skill-conditioned or no-skill view under this old policy, and define
\begin{equation}
    \mu_c(x)=\mathbb E_{y\sim p_c(\cdot\mid x)}Q^\pi(x,y).
\end{equation}
Because both views are compared at the same anchor state, they share the same state-value baseline:
\begin{equation}
    \Delta(x)=\mu_+(x)-\mu_0(x)
    =
    \mathbb E_{p_+}A^\pi(x,y)-\mathbb E_{p_0}A^\pi(x,y).
\end{equation}
Thus, the skill-minus-no-skill gap is a branchwise local advantage difference, not a raw trajectory preference.

\begin{proposition}[First-order local improvement direction of CBSD]
\label{prop:main-local-cbsd}
Let $h$ and $\ell$ be the higher- and lower-value branches at context $x$, selected by $\mu_h(x)>\mu_\ell(x)$, and let $\delta_x=\mu_h(x)-\mu_\ell(x)>0$.
Consider an idealized small update that moves the lower-value branch toward the higher-value branch,
\begin{equation}
    p_\ell^\eta(\cdot\mid x)=(1-\eta)p_\ell(\cdot\mid x)+\eta p_h(\cdot\mid x),
    \qquad \eta\in[0,1].
\end{equation}
Then the local advantage surrogate improves by
\begin{equation}
    \mathbb E_{p_\ell^\eta}A^\pi(x,y)-\mathbb E_{p_\ell}A^\pi(x,y)=\eta\delta_x.
\end{equation}
If the induced full-policy update $\pi^\eta$ is trust-region controlled so that state-distribution shift is bounded by $C\eta^2$ for some constant $C>0$, the performance-difference view~\citep{kakade2002approximatelyoptimal,schulman2015trustregionpolicyoptimization} yields
\begin{equation}
    J(\pi^\eta)-J(\pi)
    \ge
    \frac{d_\pi(x)}{1-\gamma}\eta\delta_x
    -
    C\eta^2,
\end{equation}
where $d_\pi(x)$ is the discounted visitation mass of $x$.
\end{proposition}

Proposition~\ref{prop:main-local-cbsd} is a first-order local guarantee: with correctly estimated branch ordering and a small update, lower-to-higher distillation improves the old-policy local advantage surrogate and contributes a positive first-order term to the return lower bound.
The token-level CBSD loss can be viewed as a KL-proximal stochastic approximation to this interpolation, using a stop-gradient higher-return reference and a nonnegative gap-dependent weight $w_e$.
For mixed-sign gaps, fixed skill-to-no-skill distillation can be locally harmful when $\Delta(x)<0$; CBSD flips the teacher direction and keeps the signed local update positive.
Appendix~\ref{app:theory-cbsd} proves the proposition and further shows same-anchor credit, fixed-direction failure, and conditional KL contraction toward a local optimal response.

\section{Experiments}
\label{sec:experiments}

\subsection{Experimental Setup}
\label{subsec:exp_setup}

\begin{table*}[t]
    \centering
    \caption{Main performance on \textsc{ALFWorld}, \textsc{WebShop}, and \textsc{Search-QA}. We report success rate (\%) for \textsc{ALFWorld}, Score/Succ (\%) for \textsc{WebShop}, and accuracy (\%) for \textsc{Search-QA}. For \textsc{ALFWorld}, \emph{Succ} is the overall success rate rather than an unweighted mean over subtasks, and for \textsc{Search-QA}, \emph{Avg} is the mean accuracy over the seven datasets. For each backbone, \textbf{bold} and \underline{underline} mark the best and second-best values in the highlighted primary columns. Entries marked $\dagger$ are copied from the original papers, entries marked $^*$ are copied from SDAR~\citep{lu2026selfdistilledagenticreinforcementlearning}.}
    \label{tab:main_results}
    \scriptsize
    \setlength{\tabcolsep}{1.5pt}
    \renewcommand{\arraystretch}{1.05}
    \resizebox{\textwidth}{!}{%
    \begin{tabular}{lccccccYcYcccccccY}
        \toprule
        & \multicolumn{7}{c}{\textbf{\textsc{ALFWorld}}} & \multicolumn{2}{c}{\textbf{\textsc{WebShop}}} & \multicolumn{8}{c}{\textbf{\textsc{Search-QA}}} \\
        \cmidrule(lr){2-8} \cmidrule(lr){9-10} \cmidrule(lr){11-18}
        \textbf{Method} & \textbf{Pick} & \textbf{Look} & \textbf{Clean} & \textbf{Heat} & \textbf{Cool} & \textbf{Pick2} & \cellcolor{white}\textbf{Succ} & \textbf{Score} & \cellcolor{white}\textbf{Succ} & \textbf{NQ} & \textbf{Triv} & \textbf{Pop} & \textbf{Hotp} & \textbf{2Wk} & \textbf{MuS} & \textbf{Bam} & \cellcolor{white}\textbf{Avg} \\
        \midrule
        \rowcolor{groupgray} \multicolumn{18}{>{\columncolor{groupgray}}l}{\textit{Qwen2.5-7B-Instruct}} \\
        GRPO$^*$ & 91.2 & 87.5 & 96.2 & 81.0 & 65.0 & 57.9 & 81.2 & 80.9 & 72.6 & 45.1 & 63.7 & 44.0 & 43.6 & 43.2 & 16.8 & 37.6 & 42.0 \\
        GiGPO$^\dagger$ & 97.7 & 82.7 & 98.8 & 83.7 & 89.3 & 79.2 & 90.8 & 84.4 & 72.8 & 46.4 & 64.7 & 46.1 & 41.6 & 43.6 & 18.9 & 68.9 & 47.2 \\
        SkillRL$^\dagger$ & 97.9 & 71.4 & 90.0 & 90.0 & 95.5 & 87.5 & 89.9 & 85.2 & 72.7 & 45.9 & 63.3 & 45.9 & 43.2 & 40.3 & 20.2 & 73.8 & 47.5 \\
        D2Skill$^\dagger$ & 93.8 & 77.8 & 94.7 & 95.0 & 95.5 & 72.0 & 87.8 & 91.1 & 80.5 & 48.7 & 63.4 & 44.9 & 47.5 & 43.7 & 21.0 & 67.7 & 48.1 \\
        SkillC$^\dagger$ & 88.5 & 78.6 & 91.2 & 94.7 & 100.0 & 88.9 & 90.6 & 85.6 & 74.0 & - & - & - & - & - & - & - & - \\
        SAPO$^\dagger$ & 98.7 & 73.9 & 98.1 & 92.6 & 85.0 & 89.2 & \underline{92.2} & 90.5 & 78.1 & 48.4 & 62.9 & 46.7 & 45.0 & 45.2 & 18.3 & 46.4 & 44.7 \\
        RLSD$^*$ & 100.0 & 87.5 & 92.3 & 58.8 & 80.0 & 65.2 & 82.0 & 87.4 & 77.3 & 46.8 & 63.0 & 44.4 & 45.5 & 48.9 & 21.5 & 73.0 & \textbf{49.0} \\
        Skill-SD$^*$ & 93.9 & 93.8 & 90.9 & 100.0 & 69.2 & 68.4 & 85.1 & 86.1 & 76.5 & 47.1 & 64.5 & 47.8 & 44.2 & 42.1 & 20.2 & 69.0 & 47.8 \\
        SDAR$^*$ & 94.7 & 75.0 & 100.0 & 86.7 & 68.2 & 78.9 & 85.9 & 89.4 & \underline{82.8} & 46.3 & 63.5 & 48.2 & 43.8 & 48.4 & 19.6 & 73.0 & \textbf{49.0} \\
        Skill1 & 100.0 & 83.3 & 89.5 & 100.0 & 88.5 & 85.0 & \underline{92.2} & 86.2 & 77.3 & 46.8 & 46.3 & 47.8 & 43.7 & 39.3 & 18.2 & 70.6 & 44.7 \\
        \rowcolor{ucobblue} \textbf{\textsc{UCOB}} & 93.8 & 90.9 & 94.1 & 95.8 & 92.9 & 89.5 & \cellcolor{ucobmetric}\textbf{93.0} & 91.8 & \cellcolor{ucobmetric}\textbf{85.9} & 46.4 & 65.6 & 46.7 & 45.7 & 42.9 & 19.2 & 70.6 & \cellcolor{ucobmetric}\underline{48.2} \\
        \midrule
        \rowcolor{groupgray} \multicolumn{18}{>{\columncolor{groupgray}}l}{\textit{Qwen2.5-3B-Instruct}} \\
        GRPO$^*$ & 91.2 & 62.5 & 96.2 & 61.9 & 65.0 & 47.4 & 75.0 & 79.8 & 63.3 & 39.3 & 60.6 & 41.1 & 37.4 & 34.6 & 15.4 & 26.4 & 36.4 \\
        GiGPO & 94.3 & 72.7 & 93.8 & 100.0 & 64.3 & 89.5 & 85.9 & 83.8 & 70.3 & 43.0 & 60.5 & 46.9 & 36.2 & 35.5 & 12.1 & 64.1 & 42.6 \\
        RLSD$^*$ & 87.9 & 75.0 & 90.9 & 75.0 & 73.1 & 68.4 & 79.7 & 84.4 & 66.4 & 41.5 & 58.6 & 42.3 & 40.4 & 40.2 & 16.8 & 66.9 & 43.8 \\
        Skill-SD$^*$ & 88.2 & 50.0 & 96.2 & 52.4 & 65.0 & 57.9 & 73.4 & 75.9 & 64.0 & 44.4 & 60.4 & 44.0 & 39.5 & 40.4 & 15.4 & 64.9 & \underline{44.1} \\
        SDAR$^*$ & 97.1 & 62.5 & 100.0 & 61.9 & 75.0 & 84.2 & 84.4 & 85.0 & 68.0 & 44.8 & 58.1 & 44.3 & 38.6 & 36.2 & 15.7 & 66.1 & 43.4 \\
        StepOPSD$^\dagger$ & 97.1 & 66.7 & 87.0 & 79.1 & 78.9 & 95.0 & 83.6 & - & - & 45.0 & 61.6 & 46.2 & 39.5 & 39.5 & 14.4 & 65.3 & \textbf{44.5} \\
        Skill1 & 96.7 & 76.9 & 96.3 & 93.8 & 77.3 & 95.0 & \underline{90.6} & 84.1 & \underline{75.0} & 44.2 & 60.3 & 44.8 & 39.2 & 35.9 & 13.1 & 64.9 & 43.2 \\
        \rowcolor{ucobblue} \textbf{\textsc{UCOB}} & 93.8 & 93.3 & 95.8 & 94.1 & 91.0 & 78.9 & \cellcolor{ucobmetric}\textbf{92.2} & 84.5 & \cellcolor{ucobmetric}\textbf{78.1} & 45.0 & 61.4 & 47.0 & 37.2 & 35.8 & 12.5 & 66.2 & \cellcolor{ucobmetric}43.6 \\
        \midrule
        \rowcolor{groupgray} \multicolumn{18}{>{\columncolor{groupgray}}l}{\textit{Qwen3-1.7B}} \\
        GRPO$^*$ & 71.1 & 41.7 & 36.4 & 40.0 & 31.8 & 31.6 & 46.1 & 67.3 & 38.3 & 40.0 & 58.9 & 43.5 & 35.4 & 30.3 & 12.0 & 65.7 & 40.8 \\
        GiGPO & 96.7 & 69.3 & 88.9 & 62.5 & 59.1 & 85.0 & \underline{79.7} & 70.0 & 45.3 & 41.1 & 59.3 & 45.1 & 34.1 & 32.7 & 10.5 & 64.5 & 41.0 \\
        RLSD$^*$ & 50.0 & 37.5 & 61.5 & 19.0 & 50.0 & 21.1 & 42.2 & 74.0 & 50.8 & 38.6 & 57.3 & 43.0 & 34.5 & 34.1 & 11.5 & 65.3 & 40.6 \\
        Skill-SD$^*$ & 52.9 & 37.5 & 69.2 & 42.9 & 60.0 & 36.8 & 52.3 & 81.8 & 53.9 & 39.1 & 57.5 & 45.4 & 34.8 & 34.1 & 10.7 & 64.1 & 40.8 \\
        SDAR$^*$ & 73.5 & 25.0 & 76.9 & 33.3 & 40.0 & 36.8 & 53.9 & 76.8 & 58.6 & 39.7 & 58.9 & 45.3 & 35.9 & 35.5 & 12.6 & 65.3 & 41.9 \\
        StepOPSD$^\dagger$ & 64.7 & 44.4 & 56.5 & 60.9 & 42.1 & 55.0 & 56.3 & - & - & 40.5 & 59.4 & 44.4 & 37.1 & 32.0 & 11.6 & 64.1 & 41.3 \\
        Skill1 & 71.4 & 54.5 & 59.4 & 72.7 & 60.0 & 73.7 & 65.6 & 83.2 & \underline{61.7} & 42.7 & 59.8 & 46.0 & 35.4 & 34.8 & 12.0 & 65.7 & \textbf{42.3} \\
        \rowcolor{ucobblue} \textbf{\textsc{UCOB}} & 96.2 & 85.7 & 95.5 & 84.2 & 81.5 & 88.8 & \cellcolor{ucobmetric}\textbf{89.1} & 91.5 & \cellcolor{ucobmetric}\textbf{79.7} & 44.1 & 59.6 & 45.7 & 36.8 & 31.9 & 12.2 & 64.5 & \cellcolor{ucobmetric}\underline{42.1} \\
        \bottomrule
    \end{tabular}%
    }
    \vspace{-1.0em}
\end{table*}

\paragraph{Benchmarks and metrics.}
Following the evaluation protocol of SDAR~\citep{lu2026selfdistilledagenticreinforcementlearning}, we evaluate on \textsc{ALFWorld}~\citep{shridhar2020alfworld}, \textsc{WebShop}~\citep{yao2022webshop}, and \textsc{Search-QA}~\citep{jin2025searchr1}.
For \textsc{ALFWorld}, we report success rate on six task types: Pick, Look, Clean, Heat, Cool, and Pick2, together with the overall success rate across evaluation episodes.
For \textsc{WebShop}, we report task score and success rate on the 128-task evaluation set.
For \textsc{Search-QA}, we report answer accuracy on NQ, TriviaQA, PopQA, HotpotQA, 2Wiki, MuSiQue, and Bamboogle, where the first three are single-hop and the latter four are multi-hop search datasets.

\paragraph{Models and baselines.}
We use Qwen2.5-7B-Instruct and Qwen2.5-3B-Instruct~\citep{qwen2025qwen25technicalreport}, together with Qwen3-1.7B~\citep{yang2025qwen3technicalreport}, as the main backbones.
The compared methods fall into three groups.
\textbf{No-skill RL baselines} optimize the policy without retrieved skills, including GRPO and GiGPO~\citep{feng2025groupingrouppolicyoptimizationllm}.
\textbf{Skill-augmented agent methods} retrieve, maintain, or evolve reusable skills during training, including SkillRL~\citep{xia2026skillrlevolvingagentsrecursive}, D2Skill~\citep{tu2026dynamicdualgranularityskillbank}, SkillC~\citep{lin2026skillclearningautonomousskill}, SAPO~\citep{zhang2026coevolvingskillgenerationpolicy}, and Skill1~\citep{shi2026skill1unifiedevolutionskillaugmented}.
\textbf{Self-distillation methods} add auxiliary teacher-student supervision, including GRPO+OPSD~\citep{zhao2026selfdistilledreasoner}, RLSD~\citep{yang2026rlsd}, Skill-SD~\citep{wang2026skillsdskillconditionedselfdistillation}, SDAR~\citep{lu2026selfdistilledagenticreinforcementlearning}, and StepOPSD~\citep{zhang2026stepopsdstepawareonlinepreference}.
Together, these baselines cover the main comparison axes of \textsc{UCOB}: skill-free policy optimization, skill memory evolution, and self-distillation with fixed or adaptive supervision.
Unless otherwise specified, \textsc{UCOB} uses the same task splits, rollout group size, and evaluation budget as the corresponding baselines.
We provide the main training hyperparameters and dynamic-memory settings in Appendix~\ref{app:impl-details}.

\subsection{Main Results}
\label{subsec:main_results}

Table~\ref{tab:main_results} shows that \textsc{UCOB} consistently improves the primary agentic success metrics across model scales.
On \textsc{ALFWorld}, it achieves the best overall success rate for all three backbones, reaching \(93.0\), \(92.2\), and \(89.1\) on Qwen2.5-7B, Qwen2.5-3B, and Qwen3-1.7B, respectively.
On \textsc{WebShop}, \textsc{UCOB} also gives the best success rate for every backbone (\(85.9\), \(78.1\), and \(79.7\)) and top or near-top scores, indicating gains in both sparse success and graded task reward.
The gains are most pronounced on the compact Qwen3-1.7B model: relative to Skill1, \textsc{UCOB} improves \textsc{ALFWorld} from \(65.6\) to \(89.1\) and \textsc{WebShop} from \(61.7\) to \(79.7\), giving \(+23.5\) and \(+18.0\) point gains; relative to SDAR, the gains are \(+35.2\) and \(+21.1\) points.
The trend also holds on larger backbones, where \textsc{UCOB} surpasses the strongest prior result by \(0.8/3.1\) points on Qwen2.5-7B and \(1.6/3.1\) points on Qwen2.5-3B for \textsc{ALFWorld}/\textsc{WebShop}.
On \textsc{Search-QA}, \textsc{UCOB} remains competitive rather than uniformly dominant, staying within \(0.8\), \(0.9\), and \(0.2\) points of the best average accuracy across the three backbones.
Overall, credit-aware bidirectional self-distillation is most beneficial in multi-turn agentic environments while preserving strong search-based QA performance.

The gains are not purely determined by model size, but also by the task bottleneck and the backbone's ability to calibrate retrieved skills against current evidence.
On \textsc{ALFWorld} and \textsc{WebShop}, smaller models leave larger headroom in planning, state tracking, and error recovery; the dual-granularity skill memory provides an external experience scaffold, while CBSD converts sparse return differences into local, direction-adaptive supervision.
This makes the compensation particularly visible for Qwen3-1.7B.
In contrast, \textsc{Search-QA} depends more on search quality, evidence reading, and answer synthesis than on repeatedly reused state-level interaction skills.
Its gains are therefore less tied to model size, and may also depend on whether the backbone can use retrieved skills as auxiliary guidance rather than over-following them as fixed templates.

\subsection{Ablation Study}
\label{subsec:ablation}

\begin{figure*}[t]
    \centering
    \includegraphics[width=\textwidth]{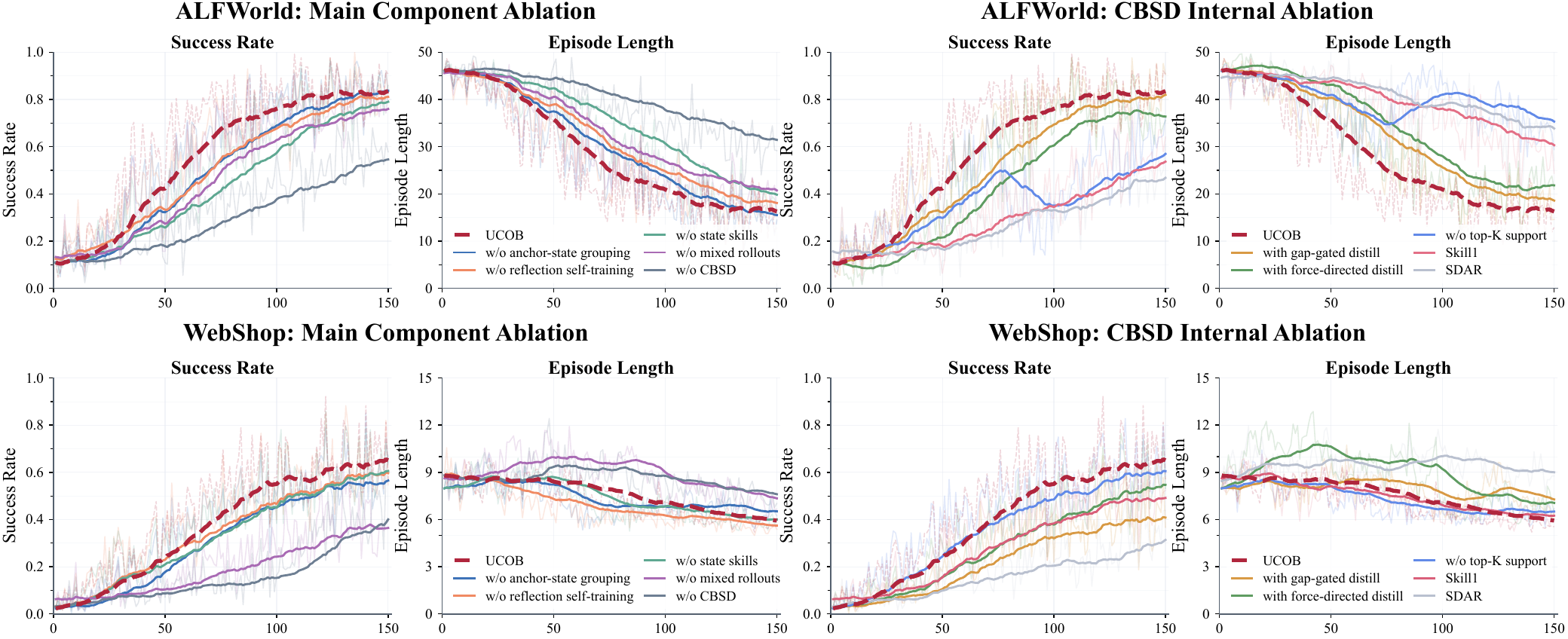}
    \caption{Ablation study under the Qwen3-1.7B backbone on \textsc{ALFWorld} and \textsc{WebShop}.}
    \label{fig:ablation-main}
    \vspace{-1em}
\end{figure*}

\paragraph{Main component ablation.}
\label{sec:ablation}
Figure~\ref{fig:ablation-main} (left) ablates the major components of \textsc{UCOB} under the Qwen3-1.7B backbone.
We report both success rate and episode length, and interpret length together with success: shorter episodes indicate more efficient behavior only when success is comparable.
The ``w/o'' variants remove one part of the full pipeline at a time: \textbf{(i)} \underline{\textit{w/o anchor-state grouping}} uses trajectory-level rather than local anchor-state comparison; \textbf{(ii)} \underline{\textit{w/o reflection self-training}} removes utility-weighted skill-writer training; \textbf{(iii)} \underline{\textit{w/o state skills}} keeps only the task-level skill pool; \textbf{(iv)} \underline{\textit{w/o mixed rollouts}} removes the paired skill/no-skill on-policy rollout design; and \textbf{(v)} \underline{\textit{w/o CBSD}} removes bidirectional distillation while keeping memory retrieval and reflection write-back.
The first three variants incur moderate degradation, indicating that anchor-state credit assignment, state-level memory, and skill-writer self-training provide complementary gains rather than redundant capacity.
The larger gap appears when the two-view learning signal is removed.
\textit{w/o mixed rollouts} eliminates paired skill/no-skill on-policy evidence, while \textit{w/o CBSD} retains memory retrieval and reflection write-back but removes the explicit cross-context distillation objective.
Their degradation shows that memory-augmented rollouts alone leave skill/no-skill comparison to sparse RL returns, instead of converting local return gaps into dense supervision.
Thus, for credit-assignment RL, \textbf{directly mixing skills into on-policy rollouts is not a substitute for an explicit local distillation signal}.

\paragraph{CBSD internal ablation.}
Figure~\ref{fig:ablation-main} (right) isolates the design choices inside CBSD.
It compares full CBSD with \textbf{(i)} \underline{\textit{gap-gated distill}}, which keeps the return-gap gate but restricts distillation to the skill-to-no-skill direction; \textbf{(ii)} \underline{\textit{force-directed distill}}, which fixes the teacher direction without local return comparison; and \textbf{(iii)} \underline{\textit{w/o top-\(K\) support}}, which removes the top-$K$ distributional support used to construct the OPD target.
The first two variants test directionality: even with a return-gap gate, one-way or fixed-direction distillation is less robust than allowing the higher-return context view to supervise the other.
This supports the central observation that skill-conditioned prompts are useful but not universally authoritative.
\textit{w/o top-$K$ support} preserves credit-aware direction selection but removes distributional support matching, indicating that top-$K$ OPD stabilizes the supervision target rather than replacing local credit assignment.
\textbf{Compared with Skill1 and SDAR, these ablated variants remain competitive and often stronger}, suggesting that the broader \textsc{UCOB} recipe already provides a strong skill-learning substrate; full CBSD provides the strongest performance stability across environments.

\paragraph{Additional ablation.}
\begin{wraptable}{r}{0.355\textwidth}
    \vspace{-1.5em}
    \centering
    \caption{Additional ablation}
    \label{tab:lightweight-sensitivity}
    \vspace{-0.5em}
    \scriptsize
    \setlength{\tabcolsep}{2pt}
    \renewcommand{\arraystretch}{1.05}
    \resizebox{\linewidth}{!}{%
    \begin{tabular}{@{}lcc@{}}
        \toprule
        \textbf{Setting} & \textbf{\textsc{ALFWorld}} & \textbf{\textsc{WebShop}} \\
        \midrule
        \rowcolor{groupgray}
        \multicolumn{3}{>{\columncolor{groupgray}}l}{\textit{CBSD coeff. $\lambda_{\mathrm{CBSD}}$}} \\
        0.01 & 83.6 & 71.1 \\
        \textbf{0.1 (Ours)} & \textbf{89.1} & \textbf{79.7} \\
        0.5 & 78.1 & 68.7 \\
        \midrule
        \rowcolor{groupgray}
        \multicolumn{3}{>{\columncolor{groupgray}}l}{\textit{Reflection backend}} \\
        Fixed memory & 82.0 & 70.3 \\
        External GPT-5.4 & 88.3 & 72.7 \\
        \textbf{Self-trained (Ours)} & \textbf{89.1} & \textbf{79.7} \\
        \bottomrule
    \end{tabular}
    }
    \vspace{-1.0em}
\end{wraptable}
Table~\ref{tab:lightweight-sensitivity} summarizes two additional ablations on key design axes.
For CBSD, a smaller \(\lambda_{\mathrm{CBSD}}\) weakens distillation, while a larger one can disrupt the balance between RL and bidirectional self-distillation; the default \(0.1\) performs best on both environments.
For skill evolution, the fixed memory initialized from SkillRL~\citep{xia2026skillrlevolvingagentsrecursive} lags behind evolving memories, showing that static external skills cannot replace continual memory updates.
External GPT-5.4 remains below self-trained reflection overall, indicating that \textsc{UCOB}'s gains do not rely on a stronger external writer.

\subsection{Analysis}
\label{sec:analysis}

\begin{figure}[t]
    \centering
    \includegraphics[width=\linewidth]{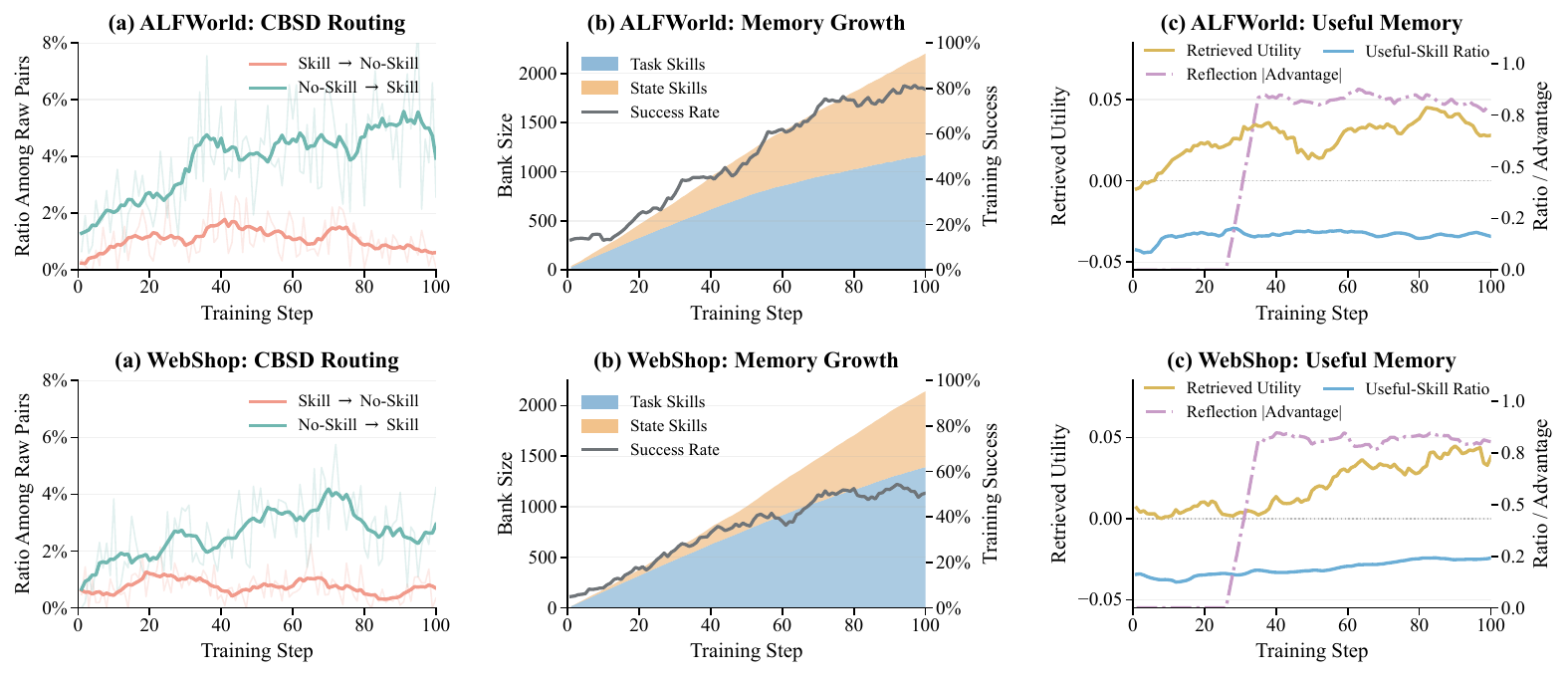}
    \caption{Mechanism and memory-evolution analysis of \textsc{UCOB} with \textsc{Qwen3-1.7B}.}
    \label{fig:analysis-mech-memory}
    \vspace{-1em}
\end{figure}

\paragraph{CBSD routing and two-view evaluation.}
\begin{wrapfigure}{r}{0.41\textwidth}
    \vspace{-1.5em}
    \centering
    \includegraphics[width=\linewidth]{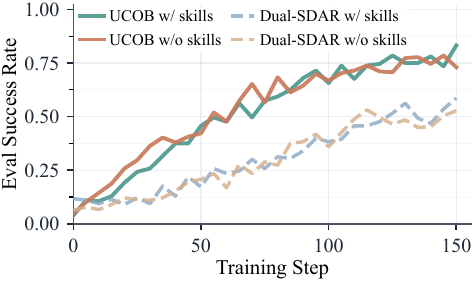}
    \caption{Skill/no-skill eval-view success with \textsc{Qwen3-1.7B}, averaged over \textsc{ALFWorld} and \textsc{WebShop}.}
    \label{fig:cbsd-view-eval}
    \vspace{-1em}
\end{wrapfigure}
Figure~\ref{fig:analysis-mech-memory}(a) tracks the teacher direction selected by CBSD during training.
In both environments, no-skill-to-skill routing occurs more often than skill-to-no-skill routing.
This indicates that reflection skills are not uniformly reliable local teachers: in many same-anchor-state comparisons, the no-skill view provides the stronger target and should correct, rather than be overwritten by, the skill-conditioned behavior.
Thus, \textbf{the skill-conditioned view is not a fixed privileged teacher}.
This routing behavior matches the theoretical view in Section~\ref{sec:theory}: when the local return gap changes sign, a fixed skill-to-no-skill teacher can become locally harmful, whereas CBSD keeps the update aligned with the higher-return view.
Figure~\ref{fig:cbsd-view-eval} further compares the two evaluation views.
Under fixed-direction dual-rollout SDAR, the skill-conditioned view is not consistently dominant.
In contrast, \textsc{UCOB} improves both skill and no-skill views, suggesting that CBSD makes the two views co-evolve through bidirectional distillation: useful skill-induced behavior is internalized, while misleading skills are corrected by the no-skill view.
\textbf{The resulting policy remains effective whether skills are provided at evaluation or not.}

\paragraph{Useful memory evolution.}
Figure~\ref{fig:analysis-mech-memory}(b) shows that both task-level and state-level skill pools grow during training, indicating that \textsc{UCOB} maintains an evolving memory rather than a fixed external repository.
Figure~\ref{fig:analysis-mech-memory}(c) examines whether this expanded memory remains useful after retrieval.
Here, \emph{Retrieved Utility} is the average utility score of retrieved skills, \emph{Useful-Skill Ratio} is the fraction of skills whose utility exceeds the initial value, and \emph{Reflection \(|\)Advantage\(|\)} measures the magnitude of the advantage signal used by the reflection writer.
Retrieved utility tends to remain positive in later training, while the useful-skill ratio stays non-trivial and the reflection signal remains active.
\textbf{The skill memory grows while its retrieved subset remains useful.}

\paragraph{Localized training cost.}
\begin{wrapfigure}{r}{0.41\textwidth}
    \vspace{-1.0em}
    \centering
    \includegraphics[width=0.98\linewidth]{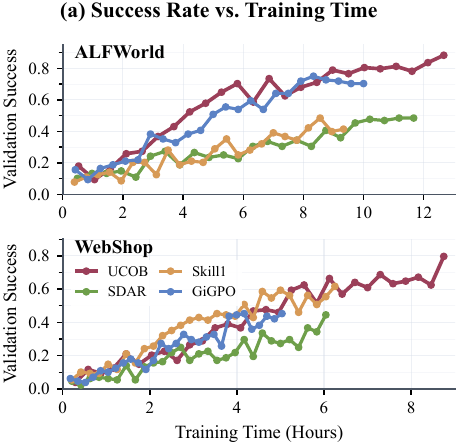}
    \par\vspace{0.30em}
    \includegraphics[width=0.98\linewidth]{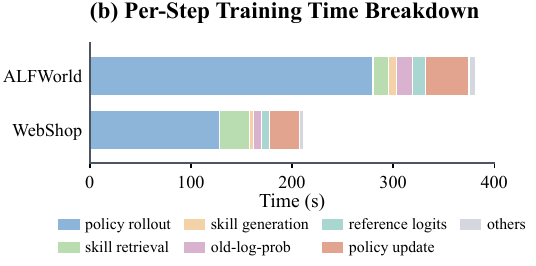}
    \caption{Cost analysis for \textsc{UCOB}.}
    \label{fig:training-cost}
    \vspace{-1.5em}
\end{wrapfigure}
Figure~\ref{fig:training-cost} evaluates the wall-clock cost effectiveness of \textsc{UCOB}.
All methods use the same \textbf{8 NVIDIA A800 GPUs} and \textsc{Qwen3-1.7B} backbone, and we measure \emph{training} time only, excluding validation and evaluation overhead.
Under this protocol, \textsc{UCOB} reaches higher validation success within a comparable wall-clock budget on both \textsc{ALFWorld} and \textsc{WebShop}.
The per-step breakdown further shows that policy rollout remains the dominant cost, while the additional CBSD-related computation is localized mainly to old-policy log-prob computation, reference-logit computation, and the policy update.
\textbf{\textsc{UCOB}'s extra training overhead is modest and localized.}

\paragraph{Continual adaptation across environments.}
Figure~\ref{fig:continual-learning} tests whether \textsc{UCOB} remains effective beyond isolated single-environment training.
We sequentially train each \textsc{Qwen3-1.7B} method on \textsc{ALFWorld}, \textsc{WebShop}, and \textsc{Search-QA}, with 150 updates per environment and evaluation on all three throughout the stream.
All RL methods use the same per-stage update budget.
As a supervised control, SFT uses 4,800 examples per environment collected by an RL-trained \textsc{Qwen3-1.7B}, with 32 examples per update to match the 150-step budget.
Recent SDFT results show that on-policy self-distillation can reduce forgetting during sequential learning from demonstrations~\citep{shenfeld2026selfdistillationenablescontinuallearning}.
Although our setting instead uses reward-driven agentic RL, we observe a related pattern: after the two subsequent stages, SDAR and \textsc{UCOB} retain \(67.1\%\) and \(82.8\%\) on \textsc{ALFWorld}, only \(2.7\) and \(4.7\) points below their respective stage-end results.
\textsc{UCOB} also preserves \textsc{WebShop} performance after \textsc{Search-QA} training (\(75.8\%\rightarrow77.3\%\)) and reaches \(45.2\%\) on \textsc{Search-QA}.
Its final three-environment average is \(68.4\%\), compared with \(57.0\%\) for SFT and \(51.8\%\) for SDAR.
\textbf{Thus, on-policy self-distillation provides useful continual-learning behavior, while CBSD and evolving skill memory yield a stronger retention--adaptation balance.}

\begin{figure*}[t]
    \centering
    \includegraphics[width=\textwidth]{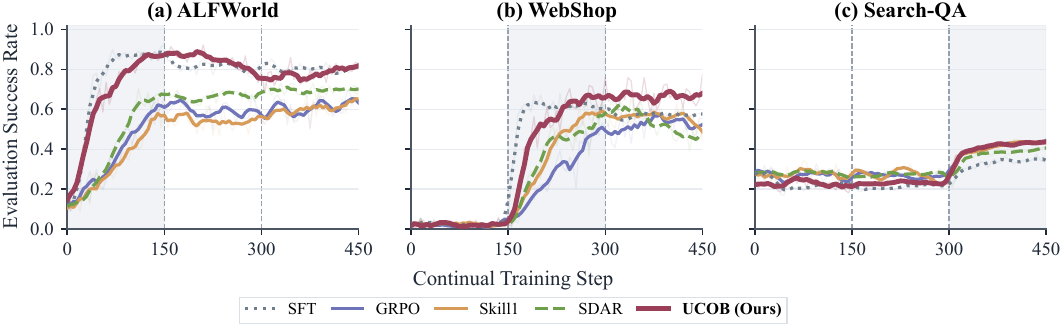}
    \caption{Continual adaptation under the \textsc{ALFWorld}$\rightarrow$\textsc{WebShop}$\rightarrow$\textsc{Search-QA} stream with \textsc{Qwen3-1.7B}. Each environment is trained for 150 steps. Dashed lines mark task transitions, and gray shading marks the active training interval in each panel.}
    \label{fig:continual-learning}
\end{figure*}

\paragraph{Case study.}
Figure~\ref{fig:case-study} shows two \textsc{ALFWorld} cases that make the bidirectional behavior concrete.
In the cellphone-to-dresser task, the no-skill view drifts to an unrelated search location, while retrieved state skills guide the skill-conditioned view to another plausible table; \textsc{UCOB} therefore selects the skill view and distills it into the no-skill view.
Conversely, in the two-watch-to-sidetable task, the retrieved state skill triggers an unnecessary attempt to take a watch that is already visible on the sidetable, while the no-skill view proceeds to place it; \textsc{UCOB} selects the no-skill view to correct the misleading skill context.
These cases show that skills are neither always trusted nor always discarded: \textbf{the teacher direction is decided by local return evidence.}

\section{Related Work}

\begin{figure}[t]
    \centering
    \includegraphics[width=.95\linewidth]{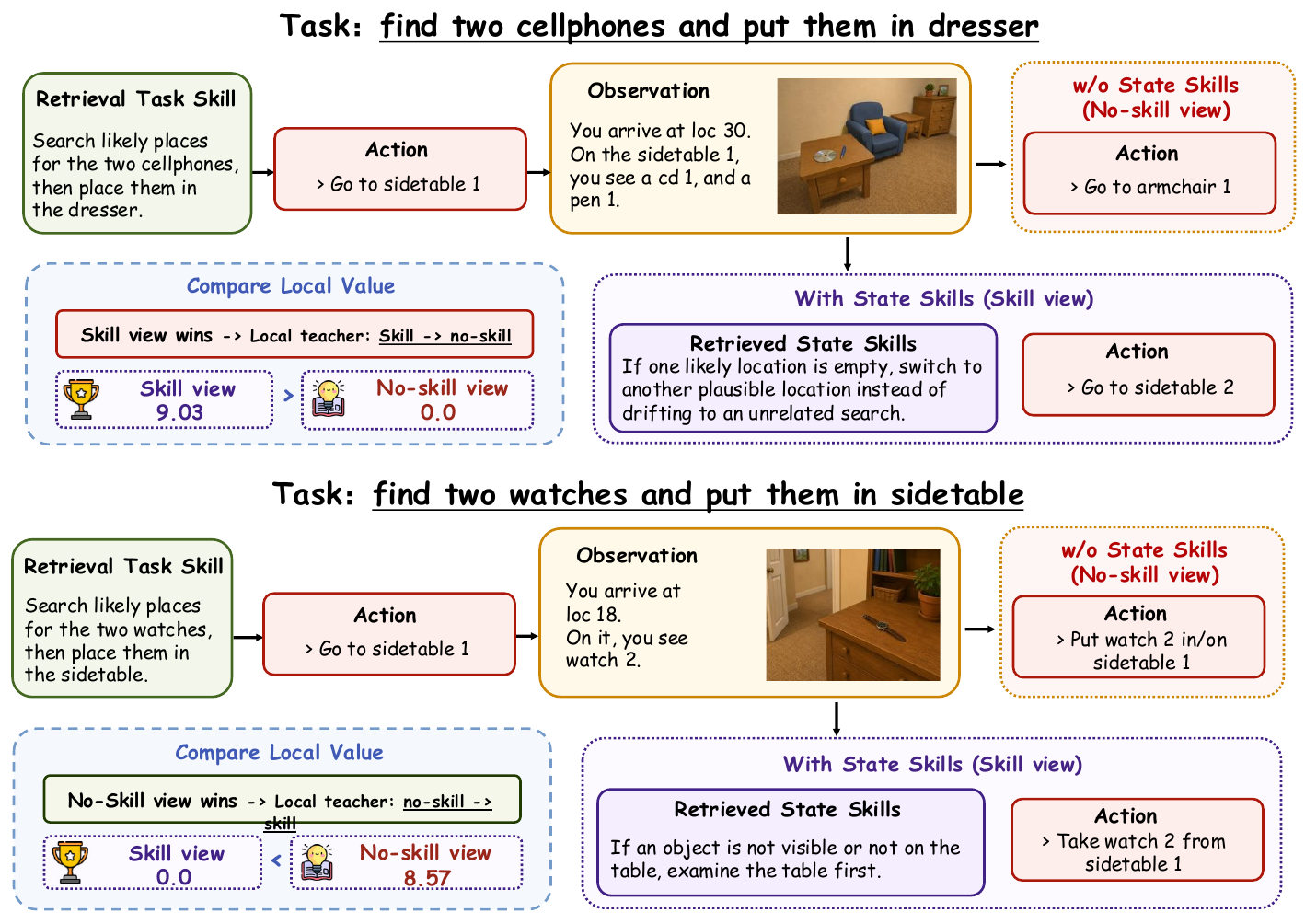}
    \caption{Case study of local teacher selection in \textsc{UCOB}.}
    \label{fig:case-study}
    \vspace{-0.5em}
\end{figure}

\paragraph{Language agents with memory and skills.}
Retrieval-augmented generation conditions language models on external information \citep{lewis2021retrievalaugmentedgenerationknowledgeintensivenlp}, and agent systems extend this idea with reasoning traces, verbal feedback, long-term memories, and reusable skills \citep{yao2023reactsynergizingreasoningacting,shinn2023reflexionlanguageagentsverbal,zhao2023expelllmagents,wang2023voyageropenendedembodiedagent,zhang2024surveymemorymechanism}.
Recent agent methods further train agents to build, curate, select, or internalize skill memories, including recursive skill-memory systems, dual-granularity skill pools, learned skill curation, evolving skill repositories, and explicit skill internalization \citep{xia2026skillrlevolvingagentsrecursive,mi2026skillpro,tu2026dynamicdualgranularityskillbank,ouyang2026skilloslearningskillcuration,zhang2026retroagentsolvingevolvingretrospective,shi2026skill1unifiedevolutionskillaugmented,vishe2026skillr1,lu2026skill0,he2026siri,lin2026skillclearningautonomousskill,zhang2026coevolvingskillgenerationpolicy,chen2026skillharness}.
These works differ in whether skills are treated mainly as retrieval-time context, evolving external memory, or behavior to be internalized by the policy.
\textsc{UCOB} addresses a complementary question: when retrieved skills are imperfect, training should decide whether to internalize the skill-conditioned behavior or correct it with evidence from the no-skill view, while evolving the skill memory itself.

\paragraph{Agentic RL and credit assignment.}
PPO and its variants remain common RL objectives for language agent training \citep{schulman2017proximalpolicyoptimizationalgorithms}.
For multiturn environments, sparse terminal rewards make trajectory-level updates inefficient, motivating agentic RL frameworks with finer-grained credit assignment, interactive guidance, or disentangled optimization \citep{xia2025ragen,feng2025groupingrouppolicyoptimizationllm,zeng2025turnlevelcredit,zhu2026gagpo,li2026distillselectivehindsightdistillation,zhang2026stepopsdstepawareonlinepreference,li2026reasoningtoolusecompeteagentic,liu2026socraticpo}.
\textsc{UCOB} uses credit evidence differently: it compares skill-conditioned and no-skill rollouts within the same task and anchor state, and converts the local return gap into a distillation direction between context views.
Thus, credit assignment determines not only which action tokens to reinforce, but also which context view should serve as the local teacher.

\paragraph{On-policy self-distillation with privileged context.}
Knowledge distillation and policy distillation transfer behavior from a teacher distribution to a student policy \citep{hinton2015distillingknowledgeneuralnetwork,rusu2016policydistillation}.
On-policy self-distillation adapts this idea to student-generated rollouts, providing dense token-level supervision alongside RL and motivating recent work on signal filtering, reversed teachers, and shorter on-policy prefixes \citep{zhao2026selfdistilledreasoner,yang2026rlsd,song2026surveyopd,fu2026revisitingonpolicydistillation,li2026rethinkingopd,zhao2026rosd,kim2026rebelliousstudent,zhang2026fullrolloutsopd}.
Privileged self-distillation has also been used for multiagent self-play without external data \citep{chen2026piplay}, while Skill-SD and SDAR instantiate a privileged teacher by conditioning on retrieved skills and distilling from the skill-conditioned prompt into the no-skill prompt \citep{wang2026skillsdskillconditionedselfdistillation,lu2026selfdistilledagenticreinforcementlearning}.
However, this fixed-teacher assumption fails when skill retrieval or utilization is locally harmful.
\textsc{UCOB} instead treats the skill-conditioned and no-skill prompts as two on-policy views of the same model and selects the teacher direction from same-anchor-state returns.
It complements prior OPD-style filtering by making teacher selection itself credit-aware, internalizing useful skills and correcting misleading skill context within one bidirectional objective.


\section{Conclusion}

We studied how language agents can learn from reusable but imperfect retrieved skills.
Our observations show that skill-conditioned prompts are state-dependent teachers rather than universally reliable privileged views, so on-policy skill rollouts still require local credit assignment.
We introduced \textsc{UCOB}, a unified framework for learning to utilize and evolve agentic skills via credit-aware on-policy bidirectional self-distillation.
CBSD compares skill-conditioned and no-skill views under the same task and anchor state, selects the higher-return view as the local teacher, and thereby internalizes useful skill-induced behavior while correcting misleading skill context.
Together with dual-granularity skill memory evolution, utility-aware retrieval, and reflection self-training, \textsc{UCOB} achieves consistent gains on \textsc{ALFWorld}, \textsc{WebShop}, and \textsc{Search-QA}; analyses further show improved context views, useful evolving memories, robust adaptation across sequential environments, and localized training overhead.
Overall, this work frames skill utilization as a credit-assignment problem: agents should learn not only which skills to retrieve, but also when to trust, correct, and internalize them during on-policy training.

\bibliography{iclr2026_conference}
\bibliographystyle{iclr2026_conference}

\appendix

\section{Detailed Protocols for the Observation Study}
\label{app:obs-settings}

We give the full experimental settings and protocols behind the observation study in Section~\ref{sec:observations}, whose schematic is shown in Figure~\ref{fig:obs-settings}.
We use $P_+$ and $P_0$ to denote the skill-conditioned and no-skill prompts, respectively.
The first two protocols instantiate fixed-direction on-policy self-distillation (OPSD), where the skill-conditioned view supervises the no-skill view through $\mathcal L_{\mathrm{OPSD}}$~\citep{lu2026selfdistilledagenticreinforcementlearning,zhao2026selfdistilledreasoner}.
The third protocol is a diagnostic: it does not fix a teacher, but instead groups matched anchor states and compares return-to-go under the two prompt views.
Throughout this appendix, $\mathcal L_{\mathrm{RL}}$ denotes the online RL loss used for environment rollouts.

\paragraph{Fixed teacher without skill rollouts.}
This protocol matches the SDAR setting analyzed above~\citep{lu2026selfdistilledagenticreinforcementlearning}.
Only the no-skill prompt $P_0$ is executed in the environment, so rollout records are generated and optimized under $\mathcal L_{\mathrm{RL}}$ from the no-skill view.
The skill-conditioned prompt $P_+$ is evaluated only by a teacher forward pass on the same student-generated prefixes, providing the additional $\lambda\,\mathcal L_{\mathrm{OPSD}}$ target for $P_0$.
Thus, $P_+$ is treated as a privileged self-teacher without receiving rollout-based RL updates under its own context.
This setting tests whether a skill-conditioned teacher without skill rollouts can be both reliable and self-correcting, as shown in Figure~\ref{fig:sdar-dual-acc}.

\paragraph{Dual rollouts with fixed-direction OPSD.}
This protocol follows the dual-rollout training design of Skill-SD~\citep{wang2026skillsdskillconditionedselfdistillation}.
Each batch executes both prompt views: half of the trajectories act under $P_+$ and half act under $P_0$, so both views receive $\mathcal L_{\mathrm{RL}}$ from environment interaction.
However, the distillation direction is still fixed: $P_+$ supervises $P_0$, and only the no-skill branch receives the additional $\lambda\,\mathcal L_{\mathrm{OPSD}}$ term.
This isolates whether making the skill-conditioned teacher on-policy is sufficient.
Figure~\ref{fig:dsdar-dual-acc} shows that dual rollouts reduce exposure mismatch, but do not remove teacher ambiguity because the fixed teacher can still be locally weaker than the no-skill view.

\paragraph{Anchor-state value-gap diagnostic.}
This protocol removes the fixed teacher direction and serves as the local return-gap diagnostic.
Using the anchor-state grouping mechanism of GiGPO~\citep{feng2025groupingrouppolicyoptimizationllm}, we group rollout records that share the same task $u$ and anchor state $\tilde{s}$; each record retains its acting branch $c_{i,t}\in\{+,0\}$ and return-to-go $G_{i,t}=\sum_{\ell\ge0}\gamma^\ell r_{i,t+\ell}$.
Within each group, we average return-to-go by branch to obtain $G_{+}(\tilde{s})$ and $G_{0}(\tilde{s})$, and use the local value gap $\Delta(\tilde{s})=G_{+}(\tilde{s})-G_{0}(\tilde{s})$ to diagnose which view is more reliable.
This comparison is well-posed because both views are scored under the same task and canonical state, isolating the local effect of skill conditioning from trajectory-level differences.
Positive gaps indicate states where the skill-conditioned view is the better local teacher, whereas negative gaps indicate states where the no-skill view is more reliable.
As shown in Figure~\ref{fig:fig3-motivation}, these gaps are mixed-sign and can switch across steps, motivating per-state teacher selection rather than a fixed skill-to-no-skill direction.

\section{Full Training Algorithm}
\label{app:full-algorithm}

Algorithm~\ref{alg:ucob-full} summarizes the full \textsc{UCOB} training loop, matching the four stages in Figure~\ref{fig:ucob-main}.

\begin{figure}[!t]
\centering
\begin{ucobalgorithm}{\textsc{UCOB}: Credit-Aware On-Policy Bidirectional Self-Distillation}
\label{alg:ucob-full}
\footnotesize
\begin{algorithmic}[1]
\Require policy \(\pi_\theta\), old policy \(\pi_{\theta_{\mathrm{old}}}\), reference policy \(\pi_{\mathrm{ref}}\), task distribution \(\mathcal D\), dual-granularity skill memory \(\mathcal M=\{\mathcal M_{\mathrm{task}},\mathcal M_{\mathrm{state}}\}\), rollout group size \(N\), retrieval budgets \(K_p^{\mathrm{mem}}\), CBSD margin \(\epsilon_{\mathrm{CBSD}}\), loss weights \((\lambda_{\mathrm{CBSD}},\lambda_{\mathrm{refl}},\beta_{\mathrm{KL}})\)
\Ensure trained policy \(\pi_\theta\), task-level skill pool \(\mathcal M_{\mathrm{task}}\), and state-level skill pool \(\mathcal M_{\mathrm{state}}\)
\State Initialize or load \(\pi_\theta\), \(\pi_{\mathrm{ref}}\), \(\mathcal M_{\mathrm{task}}\), and \(\mathcal M_{\mathrm{state}}\)
\For{training step \(r=1,\ldots,T\)}
    \State Sample a batch of task groups \(\{u\}\sim\mathcal D\)
    \For{each task group \(u\)}
        \State \textcolor{sciblue}{\textbf{Step 1: skill retrieval and view construction.}}
        \For{each encountered state \(s_t\)}
            \State Retrieve task/state skills \(M_{\mathrm{task}}(s_t),M_{\mathrm{state}}(s_t)\) with Eqs.~(\plaineqref{eq:ucb-retrieval}-\plaineqref{eq:retrieved-memory})
            \State Build the skill-conditioned view \(P_+(s_t)\) by inserting retrieved skills, and the no-skill view \(P_0(s_t)\) without retrieved skills
        \EndFor
        \State \textcolor{sciblue}{\textbf{Step 2: mixed on-policy rollouts and anchor-state grouping.}}
        \State Roll out \(N/2\) trajectories under \(P_+\) and \(N/2\) trajectories under \(P_0\) with the current policy
        \State Store each decision record \(\rho_{i,t}\) with branch \(c_{i,t}\in\{+,0\}\), acting/opposite prompts, response \(y_{i,t}\), and return-to-go \(G_{i,t}\)
        \State Group records by same task and canonical anchor state into \(\mathcal C(u,\tilde{s})\) as in Eq.~(\plaineqref{eq:anchor-group})
    \EndFor
    \State \textcolor{scired}{\textbf{Step 3 (core): credit-aware bidirectional self-distillation.}}
    \State Initialize the accepted pair set \(\mathcal P_{\mathrm{CBSD}}\leftarrow\emptyset\)
    \For{each anchor-state group \(\mathcal C(u,\tilde{s})\)}
        \State Select the highest-return local record \(\rho^\star\), its branch \(c^\star\), and response \(y^\star\) by Eq.~(\plaineqref{eq:credited-reference})
        \For{each opposite-view record \(\rho^{\mathrm{opp}}\in\mathcal C(u,\tilde{s})\) with \(c_{\rho^{\mathrm{opp}}}=\bar c^\star\)}
            \State Compute \(\Delta_e=G_{\rho^\star}-G_{\rho^{\mathrm{opp}}}\)
            \If{\(\Delta_e\le\epsilon_{\mathrm{CBSD}}\)}
                \State Ignore this pair because the local return evidence is weak
            \ElsIf{\(c^\star=+\)}
                \State \textcolor{sciblue}{Internalize useful skill behavior:} add a pair \(e\) that distills \(P_+\rightarrow P_0\) with weight from \(\Delta_e\)
            \Else
                \State \textcolor{scired}{Correct misleading skill context:} add a pair \(e\) that distills \(P_0\rightarrow P_+\) with weight from \(\Delta_e\)
            \EndIf
        \EndFor
    \EndFor
    \State Compute token-level top-\(K\) targets, gap weights, confidence gates, and \(\mathcal L_{\mathrm{CBSD}}\) as in Eq.~(\plaineqref{eq:cbsd-objective})
    \State \textcolor{scigreen}{\textbf{Step 4: dynamic skill evolution and joint optimization.}}
    \State Write/refine task-level skills from rollout-group reflection, and state-level skills from high-vs-low anchor-state contrasts
    \State \textcolor{scigreen}{Maintain skill memory:} deduplicate, update utility by Eq.~(\plaineqref{eq:utility-update}), and prune weak/redundant skills
    \State Self-train the reflection writer with utility-derived advantages using Eq.~(\plaineqref{eq:reflection-loss})
    \State Update \(\pi_\theta\) with the joint objective in Eq.~(\plaineqref{eq:ucob-objective})
\EndFor
\State \Return \(\pi_\theta,\mathcal M_{\mathrm{task}},\mathcal M_{\mathrm{state}}\)
\end{algorithmic}
\end{ucobalgorithm}
\end{figure}

\section{Implementation Details}
\label{app:impl-details}

\paragraph{Optimization and rollout.}
For \textsc{ALFWorld} and \textsc{WebShop}, \textsc{UCOB} uses GiGPO-style group advantage estimation with rollout group size $8$, training batch size $16$, validation batch size $128$, and discount factor $\gamma=0.95$.
We train the actor with learning rate $10^{-6}$, PPO mini-batch size $64$, per-GPU micro-batch size $8$, low-variance KL coefficient $0.01$, entropy coefficient $0$, and invalid-action penalty coefficient $0.1$.
The prompt and response length limits are $8192$ and $512$ tokens, respectively.
Rollouts are generated with vLLM using top-$p=1.0$ and maximum model length $10240$; validation uses temperature $0.4$ and top-$p=1.0$.
The maximum environment horizon is $50$ for \textsc{ALFWorld} and $15$ for \textsc{WebShop}.
Main runs use $8$ GPUs, train for $150$ training steps, validate every $5$ steps, and keep at most $5$ checkpoints.

\paragraph{CBSD.}
CBSD is enabled on mixed skill/no-skill rollouts and uses the opposite context view as the distillation target selected by same-task, same-anchor-state returns.
We use turn-level, response-scope top-$K$ OPD with $K_{\mathrm{tok}}=32$, $\lambda_{\mathrm{CBSD}}=0.1$, return-gap margin $\epsilon_{\mathrm{CBSD}}=0.05$, gap temperature $\tau_{\mathrm{CBSD}}=0.2$, maximum gap weight $w_{\max}=2.0$, and at most two accepted teacher-target pairs per anchor-state group.
The confidence gate is enabled with coefficient $5.0$.
CBSD supervision is applied on the action span, with action-token weight $2.0$.

\paragraph{Dynamic memory and reflection.}
The dual-granularity skill memory maintains separate task-level and state-level pools and retrieves from both pools during training and evaluation.
For each prompt, we retrieve top-$3$ task skills and top-$3$ state skills after top-$10$ sentence-transformer candidate filtering, using state-aware task/state retrieval queries and UCB scoring with exploration scale $0.1$.
The relevance threshold is $0.4$, the context budget for retrieved skills is $4096$ tokens, and skill utilities are updated with the baseline-relative rule using EMA rate $\beta_U=0.2$ for both pools.
Reflection writes one task-level item per rollout group and one state-level item per qualifying anchor-state contrast, with minimum local gap $0.05$, maximum reflection response length $2048$, temperature $1.0$, and top-$p=1.0$.
Reflection self-training starts after a $30$-step warmup in the Qwen3-1.7B runs, samples at most $16$ task-level and $16$ state-level memories per step, keeps memories with absolute utility at least $0.02$, computes pool-normalized advantages with clipping range $[-2,2]$, and uses loss weight $\lambda_{\mathrm{refl}}=1.0$.

\paragraph{\textsc{Search-QA} runs.}
For \textsc{Search-QA}, the search-agent scripts use rollout group size $8$, training batch size $128$, validation batch size $512$, prompt and response limits $4096/512$, maximum search horizon $4$, and history length $4$.
The actor learning rate is $10^{-6}$ with warmup ratio $0.1$, PPO mini-batch size $256$, per-GPU micro-batch size $16$, KL coefficient $0.001$, and invalid-action penalty coefficient $0.01$.
Search rollouts call the local retrieval server and are trained for $150$ steps on $8$ GPUs, with validation at the final training step.

\section{Local Policy-Improvement View of CBSD}
\label{app:theory-cbsd}

We analyze CBSD as a local policy-improvement operator rather than a global convergence guarantee.
The full \textsc{UCOB} system is nonstationary because the policy, retrieval distribution, and skill memory co-evolve during training.
Accordingly, we fix one policy update, one task-anchor group, and the old on-policy continuation, and ask what the CBSD direction means locally.
The argument follows the performance-difference and trust-region view of policy improvement~\citep{kakade2002approximatelyoptimal,schulman2015trustregionpolicyoptimization}, and interprets KL distillation as a small mirror-descent or advantage-weighted supervised step~\citep{tomar2020mirrordescentpolicyoptimization,peng2019advantageweightedregression}.

Let $x=(u,\tilde{s})$ denote a same-task, same-anchor-state context.
For clarity, write $p_c(y\mid x)$ for the response-level distribution induced by branch $c\in\{+,0\}$ under the old policy, where $+$ denotes the skill-conditioned view and $0$ denotes the no-skill view.
After response $y$ is produced, the continuation follows the old policy.
Define
\begin{equation}
    Q^\pi(x,y)
    =
    \mathbb E\!\left[G_t\mid u,\operatorname{anchor}(s_t)=\tilde{s},\,y_t=y\right],
    \qquad
    \mu_c(x)=\mathbb E_{y\sim p_c(\cdot\mid x)} Q^\pi(x,y),
\end{equation}
and let $A^\pi(x,y)=Q^\pi(x,y)-V^\pi(x)$.
The branch-level value gap is $\Delta(x)=\mu_+(x)-\mu_0(x)$.
The implemented objective uses sample-level accepted gaps $\Delta_e=G_{\rho^\star}-G_{\rho^{\mathrm{opp}}}$ with weight $w_e=\operatorname{clip}(\Delta_e/\tau_{\mathrm{CBSD}},0,w_{\max})$.

We first isolate the same-anchor credit identity, then use it to prove the main proposition and connect the result to the implemented token-level objective.

\begin{proposition}[Same-anchor returns estimate local advantage differences]
\label{prop:same-anchor-advantage}
For a fixed task-anchor context $x=(u,\tilde{s})$ and fixed old-policy continuation,
\begin{equation}
    \Delta(x)
    =
    \mathbb E_{y\sim p_+(\cdot\mid x)} A^\pi(x,y)
    -
    \mathbb E_{y\sim p_0(\cdot\mid x)} A^\pi(x,y).
\end{equation}
Thus, within the same anchor state, the skill-minus-no-skill return gap is a local advantage-difference signal rather than a raw trajectory-level preference.
\end{proposition}

\begin{proof}
By definition, $\mu_c(x)=\mathbb E_{y\sim p_c(\cdot\mid x)}Q^\pi(x,y)$ for each branch $c$.
Since the two branches are compared at the same $x$, the baseline $V^\pi(x)$ is shared:
\begin{align}
    \mathbb E_{p_+}A^\pi(x,y)-\mathbb E_{p_0}A^\pi(x,y)
    &=
    \left(\mathbb E_{p_+}Q^\pi(x,y)-V^\pi(x)\right)
    -
    \left(\mathbb E_{p_0}Q^\pi(x,y)-V^\pi(x)\right) \\
    &=
    \mu_+(x)-\mu_0(x)
    =
    \Delta(x).
\end{align}
Therefore, grouping by the same anchor state removes the state-value baseline and makes the return gap a local credit signal for choosing the teacher direction.
\end{proof}

\paragraph{Restatement of Proposition~\ref{prop:main-local-cbsd}.}
Let $h$ and $\ell$ be the higher- and lower-value branches at context $x$, with $\delta_x=\mu_h(x)-\mu_\ell(x)>0$.
For the idealized update $p_\ell^\eta=(1-\eta)p_\ell+\eta p_h$, the local advantage surrogate increases by $\eta\delta_x$.
If the induced full-policy update is trust-region controlled with state-distribution error bounded by $C\eta^2$, then
\begin{equation}
    J(\pi^\eta)-J(\pi)
    \ge
    \frac{d_\pi(x)}{1-\gamma}\eta\delta_x-C\eta^2.
\end{equation}

\paragraph{Proof of Proposition~\ref{prop:main-local-cbsd}.}
\begin{proof}
Let $h$ and $\ell$ denote the higher- and lower-value branches, so $\delta_x=\mu_h(x)-\mu_\ell(x)>0$.
For the idealized interpolation $p_\ell^\eta=(1-\eta)p_\ell+\eta p_h$, linearity and Proposition~\ref{prop:same-anchor-advantage} give
\begin{align}
    \mathbb E_{p_\ell^\eta}A^\pi(x,y)-\mathbb E_{p_\ell}A^\pi(x,y)
    &=
    \eta\left(
    \mathbb E_{p_h}A^\pi(x,y)-\mathbb E_{p_\ell}A^\pi(x,y)
    \right) \\
    &=
    \eta\left(\mu_h(x)-\mu_\ell(x)\right)
    =
    \eta\delta_x,
\end{align}
where the second equality uses the shared baseline in Proposition~\ref{prop:same-anchor-advantage}.
The performance-difference lemma expresses the global improvement as the discounted visitation-weighted advantage of the new policy under the old policy.
For a small trust-region update, replacing the new occupancy with the old occupancy introduces only a second-order error, yielding the stated lower bound.
\begin{equation}
    J(\pi^\eta)-J(\pi)
    \ge
    \frac{d_\pi(x)}{1-\gamma}\eta\delta_x-C\eta^2.
\end{equation}
When this lower bound is positive,
\begin{equation}
    J^\star-J(\pi^\eta)
    =
    J^\star-J(\pi)-\left(J(\pi^\eta)-J(\pi)\right)
    <
    J^\star-J(\pi),
\end{equation}
so the value gap to the optimal policy decreases.
\end{proof}

\paragraph{Connection to the implemented objective.}
The implemented token-level CBSD objective can be viewed as a KL-proximal stochastic approximation to the above interpolation.
For an accepted pair $e$ and token position $j$, the reference distribution $q_{e,j}$ is stop-gradient from the higher-return branch, while $p_{e,j}$ is the target-branch distribution.
Ignoring normalization,
\begin{equation}
    -\nabla_\theta\!\left(\chi_{e,j}\mathrm{KL}(q_{e,j}\|p_{e,j})\right)
    =
    \chi_{e,j}\,
    \mathbb E_{v\sim q_{e,j}}\!\left[\nabla_\theta\log p_{e,j}(v)\right],
    \qquad
    \chi_{e,j}=m_{e,j}w_e g_{e,j}\ge0.
\end{equation}
Thus, CBSD applies positive-weight imitation toward the locally higher-return branch, with the weight increasing in the accepted return gap.

\begin{proposition}[Bidirectionality avoids negative fixed-teacher updates]
\label{prop:bidirectionality}
If a fixed skill-to-no-skill rule is used, then for any context with $\Delta(x)<0$, the induced local surrogate change on the no-skill branch is negative:
\begin{equation}
    \mathbb E_{(1-\eta)p_0+\eta p_+}Q^\pi(x,y)
    -
    \mathbb E_{p_0}Q^\pi(x,y)
    =
    \eta\Delta(x)<0.
\end{equation}
CBSD reverses the direction in this case and instead obtains a positive local change of $\eta(-\Delta(x))$ on the skill-conditioned branch.
\end{proposition}

\begin{proof}
For a fixed skill-to-no-skill update, the target distribution becomes $(1-\eta)p_0+\eta p_+$.
By linearity of expectation, its local value change relative to $p_0$ is
\begin{equation}
    \eta\left(\mu_+(x)-\mu_0(x)\right)=\eta\Delta(x).
\end{equation}
When $\Delta(x)<0$, this is a negative-advantage update.
CBSD selects the no-skill branch as the teacher and updates the skill-conditioned branch toward $p_0$, giving
\begin{equation}
    \eta\left(\mu_0(x)-\mu_+(x)\right)=\eta(-\Delta(x))>0.
\end{equation}
The same argument shows that when $\Delta(x)>0$, the skill-conditioned branch is the improving teacher.
Therefore, bidirectional teacher selection is necessary to turn mixed-sign local gaps into signed policy-improvement directions.
\end{proof}

\begin{corollary}[Conditional contraction toward a local optimal policy]
\label{cor:optimal-contraction}
Let $\pi_x^\star$ be a local optimal response distribution at context $x$, and assume common support.
If the CBSD-selected higher-return branch is also closer to $\pi_x^\star$ than the lower-value branch by margin $\zeta_x>0$,
\begin{equation}
    D_{\mathrm{KL}}(\pi_x^\star\|p_h)
    \le
    D_{\mathrm{KL}}(\pi_x^\star\|p_\ell)-\zeta_x,
\end{equation}
then the idealized CBSD update satisfies
\begin{equation}
    D_{\mathrm{KL}}(\pi_x^\star\|p_\ell^\eta)
    \le
    D_{\mathrm{KL}}(\pi_x^\star\|p_\ell)-\eta\zeta_x.
\end{equation}
\end{corollary}

\begin{proof}
The KL divergence $D_{\mathrm{KL}}(\pi_x^\star\|\cdot)$ is convex in its second argument.
Thus,
\begin{align}
    D_{\mathrm{KL}}(\pi_x^\star\|p_\ell^\eta)
    &=
    D_{\mathrm{KL}}\!\left(\pi_x^\star\|(1-\eta)p_\ell+\eta p_h\right) \\
    &\le
    (1-\eta)D_{\mathrm{KL}}(\pi_x^\star\|p_\ell)
    +
    \eta D_{\mathrm{KL}}(\pi_x^\star\|p_h) \\
    &\le
    D_{\mathrm{KL}}(\pi_x^\star\|p_\ell)-\eta\zeta_x.
\end{align}
This corollary is conditional: return gaps alone do not imply distributional closeness to $\pi_x^\star$.
It formalizes the sense in which CBSD can move the target view closer to a local optimal policy when the higher-return branch is also a better local reference.
\end{proof}

\end{document}